%% file: paper.tex
\documentclass[11pt]{article}

\usepackage[margin=1in]{geometry}
\usepackage{amsthm,amsmath,amsfonts,amssymb}
\usepackage[numbers,sort&compress]{natbib}
\usepackage{graphicx}
\usepackage{booktabs}
\usepackage[colorlinks,citecolor=blue,urlcolor=blue]{hyperref}
\usepackage{cleveref}
\usepackage{mathtools}

\theoremstyle{plain}

\newtheorem{theorem}{Theorem}
\newtheorem{lemma}[theorem]{Lemma}
\newtheorem{proposition}[theorem]{Proposition}

\theoremstyle{definition}

\newtheorem*{example*}{Example}
\newtheorem{remark}{Remark}

\newcommand{\Mcal}{\mathcal{M}}
\newcommand{\Mcalnorm}{\Mcal_{\mathrm{norm}}}
\newcommand{\Ebb}{\mathbb{E}}
\newcommand{\R}{\mathbb{R}}

\newcommand{\Lcal}{\mathcal{L}}

\newcommand{\Z}{\mathbb{Z}}

\title{Attention is Just Another Name for Coupling?\\
        A Fast--Slow ODE Perspective on Hierarchical Pretraining}

\author{Zhengyuan Gao \\ \small Independent scholar \\ \small \texttt{zhengyuan.gao@outlook.com}}

\date{}

\begin{document}
\maketitle

\begin{abstract}
We re-interpret Transformer pretraining as a fast-slow, singularly
perturbed flow along depth, with untied weights as its non-autonomous      
feature. The linearised dynamics is a depth-ordered product of layer maps. Along a
token-homogeneous reference trajectory, the
linearised layer factorises along the eigenbasis of a frozen attention
kernel. Past a computable saturation depth, the flow factors through 
the block coarse-graining —-- in other words, 
running the layers is running the coarse variable, dually. 
Weight perturbations supported on the
decaying bundle move neither the persistent component of the
distinguished trajectory nor the frozen kernel to first order, so the
framework partitions parameter space into visible and invisible directions, 
with the cross-block coupling of the slow path sitting entirely on the visible side. 
How large a gate the slow path can carry is bounded by a stability margin. 
On the data side: if block emissions follow an exponential family, block-mean    
pooling captures all the information the slow path can use; but if neighbouring 
blocks carry no shared structure, no cross-block channel can help the prediction, 
and the gate amplitude is invisible in the prediction risk.        
Stability delimits what the architecture may do; the data decides what it will.   
\end{abstract}

\input{Sec-main-body}

\appendix
\input{Sec-appendix}

\bibliographystyle{plainnat}

\input{bib}
\end{document}

%% file: Sec-main-body.tex
\section{Overview}\label{sec:intro:overview}

This paper treats standard (and hierarchical) Transformer pretraining
as a single dynamical object: a \emph{fast--slow, singularly perturbed
flow}. The treatment is a \emph{re-interpretation} of contemporary
large language model (LLM) training.
Once the usual moving parts are written in one common ODE language, a
piece of structure that is normally invisible becomes explicit, and it
can be interrogated with the standard tools of dynamical systems and of
statistical identification.

Two operators do the work, and we call both of them \emph{couplings}.
By a \emph{coupling} we mean any operator through which one set of
degrees of freedom is allowed to act on another. The \emph{first},
\emph{fast} coupling is self-attention: each layer replaces a token's
state by a learned average over earlier tokens, i.e.\ a per-token
Markov kernel that mixes information \emph{within} the sequence. The
\emph{second}, \emph{slow} coupling acts on a compressed, per-block
view of the same trajectory: a gate writes a function of that coarse
view back into the fast stream. Many familiar designs --- block-mean
pooling, strided and compressed-sparse attention, hierarchical
recurrence --- are \emph{the same construction with different choices
of the compression operator}.

In this perspective, the structure exposes the significance of the
\emph{slow path}. Almost all of the analytical attention in the
literature is spent on the fast coupling; the second coupling is
treated as an implementation detail. We argue it is not. The two
couplings are \emph{qualitatively different} objects: they run on
different clocks (per-token versus per-block), mix different
variables, and saturate at different depths. The fast coupling mixes
information within a bounded receptive field and exhausts that mixing
after finitely many layers; the slow coupling builds a coarser,
per-block summary and feeds it \emph{back} into the fine, per-token
stream --- the efficient route to the cross-block, long-range structure
the fast path cannot reach. One coupling acts \emph{within} a scale;
the other \emph{bridges} scales. The slow coupling is
load-bearing\footnote{We use \emph{load-bearing} in a specific sense
throughout: a finding, parameter, or term is load-bearing if the
framework's conclusions \emph{structurally depend} on it --- remove it
or misspecify it and the predicted dynamics changes --- as opposed to
one that merely contributes incrementally to the loss.} precisely in
the \emph{depth-saturation} regime, where the fast path's per-token
mixing has already washed out, \emph{and} only when the data itself
carries cross-block structure. The thesis is therefore deliberately
conditional, and the conditionality is not a hedge but the content: it
is what turns an architectural preference into an identifiable,
testable claim, and it is what a conditional reading of empirical phenomena,
such as the scaling laws~\citep{kaplan2020scaling,hoffmann2022chinchilla}, rests on
(see discussions in \S\ref{sec:discussion}).

This conditionality has a direct methodological consequence. 
Because the slow path's contribution depends on what training does to 
cross-block weights --- not just on what the architecture allows --- the relevant 
analysis must handle weights that vary with depth. 
A trained Transformer does not tie its weights across layers, so
the depth flow it realises is \emph{non-autonomous}: the generator
varies with depth. We will not assume the depth variation away. The
linearised depth dynamics is a \emph{family} of layer maps
$(I + A(\ell))$, and the object that plays the role of the matrix
power is their depth-ordered product,
\[
\Phi(\ell, s) \;=\; \prod_{r=s}^{\ell-1}\bigl(I + A(r)\bigr),
\]
a two-parameter family satisfying $\Phi(\ell,s) = \Phi(\ell,r)\,
\Phi(r,s)$ --- what the non-autonomous literature calls a \emph{linear
cocycle}~\citep{coppel1978dichotomies,sackersell1978spectral,kloeden2011nonautonomous},
and the object on which every subsequent notion is defined.\footnote{A
linear cocycle is a family $\Phi(\ell,s)$, $\ell \ge s$, of linear maps
with $\Phi(s,s) = I$ and $\Phi(\ell,s) = \Phi(\ell,r)\Phi(r,s)$: the
standard non-autonomous generalisation of a matrix power.} The
classical autonomous vocabulary is replaced, throughout, by its
non-autonomous counterpart: eigenvalues by the \emph{dichotomy
spectrum}, the stable/unstable splitting by an \emph{exponential
dichotomy} with depth-dependent bundles, the fixed point by a
\emph{distinguished bounded trajectory}, and eigenvalue perturbation
bounds by the \emph{roughness} of dichotomies. The autonomous
(weight-tied) case survives inside every statement of the paper as the
specialisation in which the cocycle is a power and the dichotomy
spectrum collapses to eigenvalue magnitudes.

The development is straightforward, in three movements matching the
three things we prove.

\emph{The structure} (\S\ref{sec:std}--\S\ref{sec:settings}): a
Transformer layer is a forward-Euler step of a depth flow, and along a
token-homogeneous reference trajectory the linearised layer has the
explicit shape
\[
J(\ell) \;=\; I_T \otimes \bigl(I + M(\ell)\bigr)
\;+\; \alpha^\star \otimes V(\ell),
\]
where $\alpha^\star \in \R^{T\times T}$ is the \emph{frozen} attention
kernel on the position axis --- a Markov matrix on the causal window,
the same at every layer (\Cref{thm:factor}) --- and $M(\ell), V(\ell)
\in \R^{d\times d}$ are the linearisations of the within-token Multilayer Perceptron (MLP)
path and of the value path on the feature axis. Because
$\alpha^\star$ is depth-independent, its eigenbasis block-diagonalises
$J(\ell)$ at every layer simultaneously: each position-space
eigenvalue $\lambda_j$ opens a $d\times d$ feature-space \emph{sector
cocycle} $\Phi_j$, on which the entire depth variation of the network
lives. The $\lambda = 1$ sector is the manifold of token-homogeneous
states, exactly invariant under the nonlinear layer map, and carries
the tangent cocycle $\Phi := \Phi_1$; the fluctuation sectors
($\lambda_j < 1$) are the fast, transverse directions. The fast--slow
split is in this sense not a modelling choice but a spectral fact of
the frozen kernel, and the second coupling closes a singularly
perturbed system in the standard Tikhonov form.

\emph{The identification} (\S\ref{sec:coupling}): the framework asks
which directions in parameter space are \emph{visible} to the
asymptotic observables of the trained network. We give the question a
precise non-autonomous form. Past depth saturation, the depth flow
factors through the block summary --- running the layers is the same
as running the coarse-grained variable (\Cref{thm:duality}) --- and on
that structured limit we prove a sharp partition of parameter space:
weight changes that touch only the decaying directions of the depth
dynamics move neither the persistent component of the distinguished
trajectory nor the frozen attention kernel to first order, while
weight changes along the persistent directions, the gate amplitude,
the injection, and the compression are all visible. The slow path's
contribution sits entirely on the visible side.

\emph{The distinguished trajectory} (\S\ref{sec:fixed}): the
linearised depth dynamics is non-autonomous, so it has no equilibrium
to anchor the analysis; the replacement is the \emph{distinguished
bounded trajectory}, the unique bounded solution of the forced linear
system under an exponential dichotomy, and it plays every role the
equilibrium would have played. We compute its gate-induced
displacement (the \emph{geometric channel}: the persistent component
of the displacement, which is what reaches the output), bound the
admissible gate through the roughness of dichotomies, and close with
the data-side counterpart: under a block-level latent model with one
Markov latent per block, block-mean compression is a sufficient
statistic, and in the degenerate case of independent blocks the gate
does not enter the Bayes risk at all. Stability delimits what the
architecture \emph{may} do; the data decides what it \emph{will}.

Throughout, the analysis is local --- a linearisation along a
reference trajectory --- and is offered as a \emph{structured limit}:
exact statements about an idealised object that the trained network
approximates, with the approximation quality itself part of the
analysis (\Cref{thm:euler}, \S\ref{sec:discussion}).

\medskip
\noindent\textbf{Notation.} Token states are rows of
$X \in \R^{T\times d}$; blocks have length $P$; the layer index is
$\ell$, the block index $k$. $\Phi(\ell,s)$ denotes the tangent
cocycle, $\Phi_j(\ell,s)$ the sector cocycles, and
$\Phi^{\mathrm{full}}$ the linearised flow on the full state
$\R^{T\times d}$. An exponential dichotomy carries projectors
$P_\pm(\ell)$, bundles $V_\parallel(\ell) = \operatorname{ran}
P_+(\ell)$ and $V_\perp(\ell) = \operatorname{ran} P_-(\ell)$. The timescale ratio is
$\varepsilon = 1/L_r$ with $L_r$ the refresh period of the held
summary $\hat y$. The distinguished bounded trajectory is
written $u^\star_\varrho$.

\section{The standard Transformer}\label{sec:std}

\textbf{Attention without layers.}
Self-attention~\citep{vaswani2017attention} updates each token's hidden
state by a learned mixture of preceding tokens. For an input sequence
$X= (x_1, \ldots, x_T) \in \R^{T \times d}$ with per-token states
$x_t \in \R^d$ ($d$ the embedding dimension), one causal attention layer
computes raw scores
\begin{equation}\label{eq:attn}
  u_{ts} = \frac{q_t^\top k_s}{\sqrt{d_k}}, \qquad
  q_t = W_Q x_t, \;\; k_s = W_K x_s, \;\; v_s = W_V x_s,
\end{equation}
where $W_Q, W_K \in \R^{d_k \times d}$ and $W_V \in \R^{d_v \times d}$
are the trained parameters; the inner product forces $d_q = d_k$, and
we adopt the single-head convention $d_k = d_v = d$ throughout, so the
temperature scaling $1/\sqrt{d_k}$ ($= 1/\sqrt d$) in \eqref{eq:attn}
is the usual one.\footnote{In the multi-head implementation,
$d_k = d_v = d/h$ per head and an output projection
$W_O \in \R^{d\times d}$ recombines the heads. Nothing in the analysis
distinguishes heads, so we work single-head and write $W_V$ for the
composite value--output map $W_O W_V$.} The hidden state is then
updated via the residual connection (the same forward update at
training and inference),
\begin{equation}\label{eq:residual}
  x_t' \;=\; x_t + \sum_{s \le t} \alpha_{ts}\, v_s
  \;=\; x_t + \Ebb_{s \sim \alpha_{t,\cdot}}[v_s],
\end{equation}
where $\alpha_{t,\cdot}$ is, by construction, a row-stochastic coupling
over the causal window $\{s : s \le t\}$.

\emph{In probabilistic language}, $\alpha_{ts}$ is the probability that
position $t$ ``attends to'' position $s$; $\alpha$ is the Markov kernel
that gives us the fast coupling. This ``expectation'' reading is a
notational device, not a claim that the network samples anything: given
fixed weights and a fixed input sequence, $x_t'$ is a deterministic
weighted sum over all $s \le t$. Each row sums to one; this
\emph{row-normalisation} is what turns the raw scores into a
row-stochastic coupling, and the canonical choice is the
softmax,\footnote{Softmax is the \emph{canonical} row-normalised
score-to-kernel map, not the only one --- any positive, row-normalised
map (e.g.\ a kernelised feature map followed by normalisation) also
yields a coupling. We use softmax throughout because it is the standard
choice in the GPT family~\citep{radford2019gpt2}.}
\begin{equation}\label{eq:softmax}
  \alpha_{t,\cdot} \;=\; \mathrm{softmax}(u_{t,\cdot}), \qquad
  \alpha_{ts} \;=\; \frac{e^{u_{ts}}}{\sum_{r\le t} e^{u_{tr}}}.
\end{equation}
Geometrically this is a projection onto the probability simplex
$\Delta^{t-1}$ over the causal window: softmax is the unique maximiser
of $\langle p, u_{t,\cdot}\rangle + H(p)$ over $\Delta^{t-1}$, with $H$
the Shannon entropy. Without the simplex constraint there is no Markov
kernel and no conditional-expectation reading of
\eqref{eq:residual}.\footnote{The mirror-image operation is
\emph{column}-normalisation; alternating row- and column-normalisation
yields the doubly stochastic kernels of optimal-transport-style
attention variants. These reappear below (\Cref{rem:sink}) as the
variant for which block-mean pooling is spectrally exact.} The
$1/\sqrt{d_k}$ scaling sets the effective \emph{temperature}: a high
temperature gives near-uniform rows (diffuse mixing), a low one peaked,
near-one-hot rows (localised mixing). This single knob controls how
fast attention mixes across positions, and it reappears in
\S\ref{sec:intro:ode} through the spectrum of the frozen attention
kernel --- the quantity governing
\emph{depth-saturation}.\footnote{Both the attention weights and the
output head are softmax choices over discrete alternatives --- the
multinomial-logit (Gibbs) form of discrete-choice
theory~\citep{mcfadden1974,train2003discrete}: the inner softmax
chooses which position to attend to, the outer which token to emit. We
use this reading only as descriptive vocabulary; no result below
depends on it.}

\begin{remark}[Softmax preserves continuity]
A point that matters for the Euler treatment below: softmax is
$C^\infty$, so it \emph{preserves} the continuity of the layerwise map.
It is the \emph{argmax} (hard selection) that would break
differentiability --- but argmax never appears \emph{inside} a layer; it
enters only at decode time, when a token is sampled from the output
distribution \emph{after} the depth flow has run.
\end{remark}

\subsection{Norm scaling: a recursive projection that bounds the sublayer inputs}\label{sec:intro:renorm}

A norm-scaling step rescales each token's hidden vector to a fixed
magnitude, removing its freedom to grow or shrink in overall scale
while leaving its direction free. We write $\mathcal S$ for this
operator throughout. The original LayerNorm~\citep{ba2016layernorm}
instance subtracts the per-token mean and divides by the per-token
standard deviation; RMSNorm~\citep{zhang2019rmsnorm}, the common choice
in current LLMs, keeps only the root-mean-square rescaling.

Norm scaling is a structural device in LLMs: the residual stream of
\eqref{eq:residual} \emph{adds} a contribution at every layer, so
without a magnitude constraint the scale of the sublayer inputs drifts
with depth and deep stacks become untrainable. Formally, $\mathcal S$
maps onto a compact submanifold $\Mcalnorm \subset \R^d$ (for RMSNorm,
the sphere $\|x\|_2 = \sqrt d$ up to the learned gains): it fixes a
magnitude and leaves the direction free. Applied at every layer it
provides a uniformly bounded argument to every sublayer's computation
--- scores, values, MLP inputs, output logits --- so that every quantity
the network computes \emph{from} the state is computed from
$\mathcal S(x) \in \Mcalnorm$. What the projection does \emph{not}
bound --- the raw residual stream --- is quantified in
\Cref{prop:bounded}.

Instances of $\mathcal S$ split into two classes, and the split matters
once we differentiate. A \emph{continuous} instance (RMSNorm,
LayerNorm, group and $L^p$ norms) keeps the normalised state bounded
\emph{and} smooth in $x$. A \emph{selector} (top-$k$, sparse and
compressed-sparse attention) keeps it bounded --- it returns one of
finitely many bounded sub-states --- but not continuous: a small change
in $x$ can flip which element is selected. The linearisation of
\S\ref{sec:intro:ode} requires a Fr\'echet derivative, so the main line
of this paper assumes a continuous, $C^1$ instance of $\mathcal S$
(together with a $C^1$ activation this is the standing smoothness
requirement, stated as (A3) in \S\ref{sec:intro:ode}); selectors are
covered by the boundedness statements only, and we flag the one place
(\Cref{rem:pi}) where the distinction matters for the slow path.

\subsection{Where the norm sits: Pre-LN versus Post-LN}\label{sec:intro:preln}

There are two places to apply $\mathcal S$ relative to the residual
addition, and the choice is not cosmetic for us. \emph{Post-LN}, the
original Transformer~\citep{vaswani2017attention}, scales \emph{after}
the residual add,
$x^{(\ell)} = \mathcal{S}(x^{(\ell-1)} + \mathrm{Sublayer}(x^{(\ell-1)}))$,
re-projecting the entire stream onto $\Mcalnorm$ at every layer.
\emph{Pre-LN}~\citep{xiong2020layernorm} scales the \emph{input} of
each sublayer,\footnote{Here $\mathrm{Sublayer}(\cdot)$ stands
generically for either of the two per-layer computations ---
attention-mixing or the MLP --- since the Pre-LN/Post-LN placement
question has the same answer for both; \eqref{eq:layer_evolution}
applies this pattern twice per layer, once per sublayer, via the
half-integer bookkeeping index $(\ell-1/2)$.}
\begin{equation*}
  x^{(\ell)} = x^{(\ell-1)} + \mathrm{Sublayer}\!\bigl(\mathcal{S}(x^{(\ell-1)})\bigr),
\end{equation*}
leaving the residual stream un-scaled so that it accumulates additively
across depth. The difference is exactly the explicit/implicit one:
Pre-LN evaluates everything at the old state and \emph{adds} (an
explicit forward-Euler step); Post-LN constrains the \emph{new} state
to lie on $\Mcalnorm$, the defining feature of an implicit (projected)
scheme. Pre-LN's layerwise map is smooth wherever $\mathcal S$ is,
Post-LN's is continuous but only piecewise smooth
(\Cref{app:partI}, \Cref{lem:preln-smooth,lem:postln-cont}). We adopt
Pre-LN because its un-normalised, additive residual stream is the
cleanest state variable for the forward-Euler reading --- and,
consistently, we will nowhere claim that the stream itself lives on
$\Mcalnorm$.

In particular, under Pre-LN, the sublayer inputs are confined to
$\Mcalnorm$ (every sublayer computation reads from a uniformly bounded
argument), but the residual stream $x^{(\ell)}$ itself grows linearly
in depth (\Cref{prop:bounded}). All confinement
statements in this paper refer to the normalised observable
$\mathcal S(x)$, never to the raw stream $x$.

\subsection{A standard Transformer block, and the renormaliser}\label{sec:intro:block}

Written with an auxiliary intermediate index $(\ell-1/2)$, the block
is\footnote{The intermediate index does not name a physical sublayer.
It is an auxiliary label so that the coupling and the rescaling each
appear as \emph{one} additive change of the residual stream; the block
then reads as two consecutive forward-Euler steps rather than one
composite map.}
\begin{align}\label{eq:layer_evolution}
  z_t^{(\ell)} &= \mathcal{S}\bigl(x_t^{(\ell-1)}\bigr), \notag\\
  x_t^{(\ell-1/2)} &= x_t^{(\ell-1)} + \sum_{s \le t} \alpha_{ts}^{(\ell)}\, W_V^{(\ell)} z_s^{(\ell)}, \notag\\
  z_t^{(\ell+1/2)} &= \mathcal{S}\bigl(x_t^{(\ell-1/2)}\bigr), \notag\\
  x_t^{(\ell)} &= x_t^{(\ell-1/2)} + W_2^{(\ell)}\, \phi\bigl(W_1^{(\ell)} z_t^{(\ell+1/2)} + b_1^{(\ell)}\bigr) + b_2^{(\ell)},
\end{align}
where $\phi$ is the MLP activation (GELU, ReLU, or similar; the
treatment uses only that $\phi$ is $C^1$ with locally Lipschitz
derivative, per (A3)), and $W_1^{(\ell)} \in \R^{d_\mathrm{ff}\times d}$,
$W_2^{(\ell)} \in \R^{d\times d_\mathrm{ff}}$. All weights carry the
layer index $\ell$. The MLP acts \emph{pointwise per token} --- it mixes
no positions --- so it is the within-token, cross-feature complement of
attention (which mixes positions but is linear in the values).

Grouping each pre-normed sublayer with its residual add, the block
simplifies to one forward-Euler step,
\begin{equation}\label{eq:one-step}
  x_t^{(\ell)} = x_t^{(\ell-1)} + \mathcal{N}^{(\ell-1)}_{t}\bigl(x^{(\ell-1)}\bigr),
\end{equation}
where $\mathcal{N}^{(\ell-1)}_t$ --- the \emph{renormaliser} --- bundles
three operations: row-normalisation of the attention scores (which
\emph{defines} the (fast) coupling), norm scaling ($\mathcal S$, applied
twice per block), and the within-token MLP nonlinearity.\footnote{The
name is Euler-step bookkeeping, not a claim that the three play the
same physical role. They constrain different degrees of freedom at
different strengths: attention is \emph{competitive} (its denominator
ties all positions together --- mass given to one position is taken
from the others), $\mathcal S$ divides each token by its own statistic
and couples nothing across tokens, and the MLP bends feature directions
within a token. The three nonlinearities also live in different spaces:
the simplex in \emph{position} space, the sphere in \emph{feature}
space, and an unconstrained feature bend. We refer to them by name
whenever the distinction matters.}

\textbf{The step is bounded; the stream is not.} Because every argument
of $\mathcal N$ passes through $\mathcal S$ first, the renormaliser is
uniformly bounded even though it is not a projection --- and this is the
precise sense in which norm scaling tames depth.

\begin{proposition}[Bounded step, linearly growing stream]\label{prop:bounded}
Fix the weights of layer $\ell$ and let $L_\phi$ be the Lipschitz
constant of $\phi$ on bounded sets. Then
\[
\sup_{x} \bigl\|\mathcal N^{(\ell)}_t(x)\bigr\|
\;\le\; \|W_V^{(\ell)}\|\sqrt d
\;+\; \|W_2^{(\ell)}\|\bigl(L_\phi(\|W_1^{(\ell)}\|\sqrt d + \|b_1^{(\ell)}\|) + \|\phi(0)\|\bigr) + \|b_2^{(\ell)}\|
\;=:\; C_{\mathcal N}^{(\ell)},
\]
and consequently the Pre-LN residual stream satisfies
$\|x^{(\ell)}\| \le \|x^{(0)}\| + \sum_{r<\ell} C_{\mathcal N}^{(r)}$
for all $\ell \le L$.
\end{proposition}
\begin{proof}
See \Cref{app:partI} (\Cref{pf:bounded}).
\end{proof}

Each term in $C_{\mathcal N}^{(\ell)}$ is a per-layer bound on one
sublayer's contribution, monotone in the per-layer weight norms (linear
in each factor separately). The bound is per-step: every layer adds at
most $C_{\mathcal N}^{(\ell)}$ to $\|x^{(\ell)}\|$. Boundedness is
therefore a property of the \emph{step} (and of every normalised
observable); the \emph{stream} $x^{(\ell)}$ itself is not bounded, but
its growth is at most linear in depth whenever
$\sup_\ell C_{\mathcal N}^{(\ell)} < \infty$. All statements about
confinement or saturation in this paper are statements about
$\mathcal S(x)$.

\subsection{From Transformer to a non-autonomous flow}\label{sec:intro:ode}

\textbf{The Euler step as a definition.} The renormaliser
$\mathcal{N}^{(\ell)}_t$ of \eqref{eq:one-step} is, for each fixed
$\ell$, an ordinary (generally nonlinear) map $\R^d \to \R^d$,
depending on $\ell$ only through the trained parameters. At step size
$\Delta\ell = 1$ the forward-Euler scheme for $\dot x = f(x,\ell)$
reads $x^{(\ell)} = x^{(\ell-1)} + f(x^{(\ell-1)}, \ell-1)$, which is
\emph{literally} \eqref{eq:one-step} once we \emph{define}
\begin{equation}\label{eq:f-def}
  f(x,\ell) \;:=\; \mathcal{N}^{(\ell)}(x).
\end{equation}
Here $\ell \in \R^{+}$ ranges over continuous depth-time; the
network's actual layers are its integer values. The identification
\eqref{eq:f-def} itself carries no error; the two error-bearing steps
--- linearising, and Euler versus exact --- are quantified below, each at
the point it is needed.

\textbf{The flow is non-autonomous.} Because the weights vary with
depth, $f(\cdot,\ell)$ is a different map at every layer: the flow is
\emph{non-autonomous} --- a trained standard stack ties nothing across
depth, and the continuous-time attention models of the mean-field
literature likewise carry depth-dependent weights from their very
first equation~\citep{geshkovski2023clusters}. This is the natural
setting for any mainline LLM, which never ties weights across layers.
Only in the autonomous (weight-tied) special case may one speak of
\emph{the} generator and \emph{its} spectrum. The non-autonomous
theory of linear cocycles, exponential dichotomies, and their spectra
does the same work, and is exactly fitted to products of layer
maps~\citep{coppel1978dichotomies,sackersell1978spectral,kloeden2011nonautonomous}.
What we require of the weights is only uniform boundedness along
depth.

\begin{itemize}
\item[\textbf{(A1)}] \emph{The realised generator family $\{A(\ell)\}_\ell$
is uniformly bounded along depth: $\sup_\ell \|A(\ell)\| \le a < \infty$
for some constant $a$, where $A(\ell) := M(\ell) + V(\ell)$ is the
layer-$\ell$ feature-space Jacobian defined in \eqref{eq:MV-def}
below.}
\end{itemize}

Boundedness is checkable on any trained artefact --- under
\Cref{prop:bounded}-type bounds, $a$ is controlled by the per-layer
weight norms and the $C^1$ bounds of (A3) --- and it is all the cocycle
theory asks. The weight-tied case, where every statement below
collapses to its familiar autonomous form, is carried along as a
running specialisation, not as the setting.

\textbf{The reference trajectory.} Depth saturation --- the empirical
flattening of per-layer updates in deep stacks --- reads, in this
language, as the configuration settling near a distinguished reference
solution of the flow. The reference we want is a \emph{trajectory}: in
a non-autonomous setting the depth index plays the role of time, and
we want a path that the actual states follow, not a single fixed
point. The structure of attention hands us such a trajectory for free:
token-homogeneous configurations are preserved by every layer,
whatever its weights, so the homogeneous manifold carries a canonical
family of solutions --- pick a homogeneous initial state and let the
network itself propagate it. The following elementary symmetry is the
backbone of the entire construction, and it is exact, nonlinear, and
weight-free.

\begin{lemma}[The homogeneous manifold is exactly invariant, at every layer]\label{lem:homog}
Consider the block map \eqref{eq:layer_evolution} at any layer $\ell$,
acting on a contiguous window of $T$ tokens with causal attention
confined to the window, and let
\[
\mathcal H \;:=\; \{X \in \R^{T \times d} : x_1 = \cdots = x_T\}.
\]
Then for \emph{any} choice of weights and \emph{any} row-stochastic
attention, the layer map takes a homogeneous state to a homogeneous
state: if all tokens start at the common value $x$, all tokens end at
the common value
\[
x' \;=\; x + W_V^{(\ell)}\mathcal S(x) \;+\; W_2^{(\ell)}\,\phi\bigl(W_1^{(\ell)}\,\mathcal S\bigl(x + W_V^{(\ell)}\mathcal S(x)\bigr) + b_1^{(\ell)}\bigr) + b_2^{(\ell)}.
\]
Consequently $\mathcal H$ is invariant under the block map, layer by
layer. (The right-hand side is just the renormaliser
$\mathcal N^{(\ell)}$ of \eqref{eq:one-step} evaluated at the
homogeneous configuration $(x, \ldots, x)$.)
\end{lemma}
\begin{proof}
If all tokens equal $x$, every value vector equals
$W_V^{(\ell)}\mathcal S(x)$, and each attention row is a probability
vector averaging identical values --- so every token receives the same
update regardless of $\alpha$; the MLP then acts identically on
identical inputs.
\end{proof}

It is worth pausing on \emph{what} makes the lemma true, because the
mechanism is more informative than the algebra.

\textbf{Two independent per-sublayer invariances compose.} Attention
preserves homogeneity because it is an \emph{averaging operator}:
consensus is a fixed point of every average, regardless of the
weights, the temperature, or the row-stochasticity details. The MLP
preserves homogeneity because it is \emph{pointwise per token}: the
same function applied to the same input gives the same output. Neither
fact uses any symmetry property of the layer map; it is the
combination --- averaging plus pointwise --- that closes $\mathcal H$
as an invariant set, layer by layer, under any weights. This can be
sharpened into a structural observation. For \emph{bidirectional}
attention the dynamics commutes with the permutation group on token
indices, and $\mathcal H$ is precisely the fixed-point set of that
group --- so invariance is forced by symmetry before any formula is
written. A \emph{causal} mask breaks the permutation symmetry; the
per-layer argument above shows that invariance survives anyway,
carried by a \emph{weaker and more robust mechanism} (averaging alone,
with no symmetry needed). And because the argument is per-layer,
untied weights cost nothing: $\mathcal H$ is invariant under the whole
non-autonomous flow, and the network restricted to
$\mathcal H \cong \R^d$ acts on the common token value $x$ as
\[
x^{(\ell)} \;=\; x^{(\ell-1)} \;+\; \mathcal N^{(\ell-1)}\bigl(x^{(\ell-1)}, \ldots, x^{(\ell-1)}\bigr).
\]
We fix once and for all a solution of this reduced system --- the
\emph{homogeneous reference trajectory}
$X^\star(\ell) = (x^\star(\ell), \ldots, x^\star(\ell))$ --- and
linearise along it.

\textbf{There is no existence assumption}: a reference trajectory
exists because $\mathcal H$ is invariant; what remains empirical is
only whether trained networks operate \emph{near} it. On that point
the evidence is the same body of work as before, and it is worth
restating in the non-autonomous light: the drift of deep attention
stacks toward token uniformity is documented as rank collapse for pure
attention~\citep{dong2021rank}; in the mean-field reading, clustering
to a single point is a theorem for the attractive, MLP-free
sub-architecture, and --- variationally --- the homogeneous
configurations are the global maximisers of the interaction energy
that the attractive flow increases along its own
trajectories~\citep{geshkovski2023clusters}: proximity to $\mathcal H$
is what the fast coupling is \emph{for}. We record the standing
locality requirement as:

\begin{itemize}
\item[\textbf{(A2)}] \emph{The states under analysis remain in a
neighbourhood of the reference trajectory $x^\star(\ell)$ in which the
linearisation assumptions of (A3) hold.}\footnote{This is a locality
condition --- it restricts the domain of validity, not the existence of
$x^\star$, which Lemma~\ref{lem:homog} provides weight-free.}

\item[\textbf{(A3)}] \emph{The norm-scaling operator $\mathcal S$ and the
activation $\phi$ are $C^1$ in $x$ along the reference trajectory
$x^\star(\ell)$: the Fr\'echet derivatives $D\mathcal S(x^\star(\ell))$
and $D\phi\bigl(W_1^{(\ell)}\mathcal S(x^\star(\ell)) +
b_1^{(\ell)}\bigr)$ exist and are continuous.}\footnote{This is the
smoothness requirement under which the linearisations
$V(\ell)$ and $M(\ell)$ of \eqref{eq:MV-def} are well-defined.} 
\end{itemize}

\textbf{The linearisation, uniformly factorised in depth.}
Linearising the block map along $X^\star(\ell)$ requires the first
variation of the attention term, which splits into a \emph{value}
variation (through $v_s$) and a \emph{kernel} variation (through
$\alpha_{ts}$). At a homogeneous state both collapse --- the first to a
universal kernel, the second to zero --- and, crucially, they collapse
\emph{in the same way at every layer}. Below, $D$ denotes the Fr\'echet
derivative (the Jacobian matrix on $\R^d$); the existence and
continuity of $D\mathcal S$ and $D\phi$ are guaranteed by (A3). Define
the layer-$\ell$ feature-space Jacobians along the reference,
\begin{equation}\label{eq:MV-def}
  V(\ell) := W_V^{(\ell)}\, D\mathcal S\bigl(x^\star(\ell)\bigr),
  \qquad
  M(\ell) := W_2^{(\ell)}\, D\phi(\cdot)\, W_1^{(\ell)}\, D\mathcal S\bigl(x^\star(\ell)\bigr),
\end{equation}
the linearisations of the value path and of the MLP
path.\footnote{We linearise the block as the \emph{sum} of the two
sublayer linearisations, dropping the sequential cross-term
$M(\ell)V(\ell)$; the cross-term is a product of two generator-scale
factors and is absorbed into the $O(\|\cdot\|^2)$ bookkeeping of
\Cref{thm:euler}. No statement below depends on it.}

\begin{theorem}[Sector factorisation, uniform in depth]\label{thm:factor}
Assume (A1)--(A3) and let $X^\star(\ell)$ be the homogeneous reference
trajectory. Then:
\begin{enumerate}
\item At every layer $\ell$, the attention kernel frozen at
  $X^\star(\ell)$ is the \emph{same} causal-uniform kernel
  $\alpha^\star_{ts} = 1/t$ for $s\le t$ --- independently of
  $W_Q^{(\ell)}, W_K^{(\ell)}$ and of $x^\star(\ell)$ itself --- with
  spectrum $\{1/t : t = 1,\ldots,T\}$; and the kernel variation
  contributes nothing to the linearisation at any layer, because the
  values are position-independent and each row of $\delta\alpha$ sums
  to zero. The frozen-kernel linearisation is exact to first order,
  layer by layer.
\item The differential of the layer-$\ell$ block map along the
  reference is
  \begin{equation}\label{eq:J-def}
    J(\ell) \;=\; I_T \otimes \bigl(I + M(\ell)\bigr) \;+\; \alpha^\star \otimes V(\ell),
  \end{equation}
  and since $\alpha^\star$ is the same at every layer, the eigenspaces
  $\{w_j u^\top : u \in \R^d\}$ of the \emph{position} factor
  block-diagonalise \emph{every} $J(\ell)$ simultaneously: on sector
  $j$, the depth evolution is the $d\times d$ \emph{sector cocycle}
  \[
  \Phi_j(\ell, s) \;:=\; \prod_{r=s}^{\ell-1}\bigl(I + M(r) + \lambda_j V(r)\bigr),
  \qquad
  A_j(\ell) := M(\ell) + \lambda_j V(\ell).
  \]
\item The homogeneous sector ($\lambda_1 = 1$, $w_1 = \mathbf 1$)
  carries the \emph{tangent cocycle}
  \begin{equation}\label{eq:A-def}
    \Phi(\ell, s) \;:=\; \Phi_1(\ell, s)
    \;=\; \prod_{r=s}^{\ell-1}\bigl(I + A(r)\bigr),
    \qquad
    A(\ell) := M(\ell) + V(\ell),
  \end{equation}
  which coincides with the linearisation of the reduced flow of
  \Cref{lem:homog} along the reference (up to the absorbed
  cross-term).
\end{enumerate}
\end{theorem}
\begin{proof}
(1) $q_t = W_Q^{(\ell)}\mathcal S(x^\star(\ell))$ and
$k_s = W_K^{(\ell)}\mathcal S(x^\star(\ell))$ do not depend on $t,s$,
so all scores in a row are equal and softmax returns the uniform row
on the causal window --- whatever the layer and its weights; a
triangular matrix has its diagonal as spectrum. For the variation,
$\sum_s \delta\alpha_{ts}\,v^\star = v^\star\sum_s\delta\alpha_{ts} = 0$
since rows of a row-stochastic kernel sum to one identically in the
state. (2) For $\delta X = w_j u^\top$:
$\delta X(I+M(\ell))^\top = w_j((I+M(\ell))u)^\top$ and
$\alpha^\star\delta X V(\ell)^\top = \lambda_j w_j(V(\ell)u)^\top$;
the sector subspaces do not depend on $\ell$ because $\alpha^\star$
does not. (3) is the case $\lambda_1 = 1$.
\end{proof}

\textbf{Why $\{w_j u^\top : u \in \R^d\}$? Two reasons.}
\emph{Position axis:} $w_j$ is a right eigenvector of $\alpha^\star$
($\alpha^\star w_j = \lambda_j w_j$), the one fixed matrix on the
position axis of every $J(\ell)$. Since $\alpha^\star$ does not depend
on $\ell$, the same $w_j$ block-diagonalises every $J(\ell)$
simultaneously --- exactly what a common sector structure across depth
requires. \emph{Feature axis:} the feature index is free --- $u$
ranges over all of $\R^d$ --- because no operator on the feature axis
of $J(\ell)$ is fixed (they vary with $\ell$). The eigenspace is
therefore $d$-dimensional: each position-axis eigenvector $w_j$
carries a $d$-dimensional fiber of feature vectors.

\textbf{Sector $1$ is the homogeneous manifold; the rest are
fluctuations.} Because $\alpha^\star$ is row-stochastic, its Perron
eigenvalue is $\lambda_1 = 1$ with right eigenvector
$w_1 = \mathbf 1$, so sector $1$ is
$\{\mathbf 1\,u^\top : u \in \R^d\}$ --- exactly the $T\times d$
matrices with all rows equal, i.e.\ the manifold $\mathcal H$ of
\Cref{lem:homog} reparametrised by the common row value
$u \in \R^d$. It is invariant under every $J(\ell)$
(\Cref{thm:factor}) --- the linear shadow of the exact nonlinear
invariance (\Cref{lem:homog}) --- so its dynamics is closed: the
tangent cocycle $\Phi_1$ is the depth evolution of a homogeneous
perturbation, independent of all other sectors. The remaining sectors
$j \ge 2$ have $|\lambda_j| < 1$ (Perron--Frobenius for stochastic
matrices) and decay geometrically under repeated $\alpha^\star$
applications --- the classical Markov-chain mixing
statement~\citep{norris1997markov,levin2009markov}; their cocycles
$\Phi_j$ govern the depth evolution of within-block fluctuations
around the homogeneous reference, transverse to $\mathcal H$ and
controlled by (A5) below. In the oblique eigenbasis of $\alpha^\star$
(the change of basis costs the factor $\kappa_\alpha$, its eigenbasis
condition number, in norm estimates), the full linearised flow is
block-diagonal,
\[
\Phi^{\mathrm{full}}(\ell, s) \;=\; \bigoplus_{j=1}^{T} \Phi_j(\ell, s),
\]
and two structural facts are visible at a glance: sector $1$ is closed
under depth evolution (no off-diagonal coupling to other sectors), and
the $T-1$ fluctuation sectors are independent of each other and of
sector $1$. The dichotomy analysis below can therefore proceed
\emph{sector by sector}, and the dichotomy spectrum of
$\Phi^{\mathrm{full}}$ is the union of the dichotomy spectra of
$\Phi_1, \ldots, \Phi_T$ --- with sector $1$ carrying the central (A4)
saddle and sectors $j \ge 2$ carrying the (A5) transverse contraction.

This theorem is the reason the non-autonomous program is affordable.
In general, a depth-varying linearisation offers no common
eigenstructure to work in --- each layer would rotate the
decomposition of its predecessor. The \emph{position-space
scaffolding} is rigid because the Kronecker structure of $J(\ell)$
decouples the position and feature axes: the operators acting on the
position axis ($I_T$ and $\alpha^\star$) are $\ell$-independent, so the
eigenvectors $\{w_j\}$ of $\alpha^\star$ are \emph{common} to every
$J(\ell)$; the operators acting on the feature axis
($M(\ell), V(\ell)$) are $\ell$-dependent, but the decoupling means
they cannot rotate the position-axis decomposition. All depth
variation is therefore confined to the $d\times d$ sector cocycles
$\Phi_j$ that act on the fibers.

\begin{remark}[The head of the window; doubly stochastic variants]\label{rem:sink}
The Perron left eigenvector $\pi$ ($\pi\alpha^\star = \pi$) of the
causal-uniform kernel concentrates on the first position
($\pi = e_1$): under repeated application of $\alpha^\star$ --- which
is the same matrix at every layer --- all tokens relax toward the
sequence head, the linear-regime shadow of the empirically observed
attention-sink behaviour. For doubly stochastic attention variants the
stationary vector is uniform and the block \emph{mean} is exactly the
$\lambda = 1$ coordinate; for the causal kernel the block mean
approximates it, with an error carried by the fluctuation sectors,
which decay under the contraction hypothesis introduced in
\S\ref{sec:settings}.
\end{remark}

\begin{remark}[Which object is Markov --- and which is autonomous]\label{rem:markov}
Two operators on two different spaces appear in \Cref{thm:factor},
and they differ in both respects. The row-stochastic kernel
$\alpha^\star$ on \emph{position} space is a Markov transition matrix
and an \emph{autonomous} one --- the same at every layer --- so
statements about mixing, e.g.\ that $(\alpha^\star)^\ell$ approaches
its rank-one limit geometrically, are classical Markov-chain
statements~\citep{norris1997markov,levin2009markov}. The generator
family $A(\ell)$ on \emph{feature} space is neither: it carries no
stochastic structure (it is a Jacobian; its columns obey no
conservation law) and it varies with depth. We keep the two faces
strictly apart: ``mixing'' and ``saturation'' always refer to
$\alpha^\star$; ``growth rates'' and ``trajectory'' always refer to
the cocycles.
\end{remark}

\textbf{Discretisation error.} With $\delta x$ the deviation from the
reference in the homogeneous sector, the linearised depth dynamics and
its Euler step are
\begin{equation}\label{eq:linear-block}
  \dot{\delta x} \;=\; A(\ell)\,\delta x,
  \qquad
  \delta x^{(\ell)} \;=\; \bigl(I + A(\ell-1)\bigr)\,\delta x^{(\ell-1)}
  \;=\; \Phi(\ell, 0)\,\delta x^{(0)} .
\end{equation}

\begin{theorem}[Euler is the first-order truncation of the linearised flow]\label{thm:euler}
Assume (A1) and let $\ell \mapsto A(\ell)$ be continuous on the depth
window under analysis. Then the forward-Euler step of
$\dot{\delta x} = A(\ell)\,\delta x$ with step $\Delta\ell$ satisfies
\[
\delta x(\ell + \Delta\ell) \;=\; \bigl(I + \Delta\ell\, A(\ell)\bigr)\,\delta x(\ell)
\;+\; O\!\bigl(a^2\Delta\ell^2 + \omega_A(\Delta\ell)\,\Delta\ell\bigr),
\]
where $a = \sup_\ell\|A(\ell)\|$ and $\omega_A$ is the modulus of
continuity of $A(\cdot)$; the exact step is the time-ordered
exponential. At $\Delta\ell = 1$ and $a \sim 1$ the per-layer deviation
is $O(1)$, so the Euler reading is \emph{qualitative}, not
quantitative, at mainline depths; it becomes quantitative exactly in
the small-generator (saturation) regime $a \ll 1$ with slowly varying
weights.
\end{theorem}
\begin{proof}
See \Cref{app:partI} (\Cref{pf:euler}).
\end{proof}

The Euler deviation in \Cref{thm:euler} quantifies one of the two
error sources that the linearised depth dynamics carries.
\emph{Euler-versus-exact error} (the subject of the theorem): the
layer map $(I + A(\ell-1))$ is a forward-Euler step of the linearised
flow, with error $O(a^2\Delta\ell^2 + \omega_A(\Delta\ell)\Delta\ell)$;
at $\Delta\ell = 1$ this is $O(1)$ at mainline depths, so the Euler
reading is qualitative there and becomes quantitative only in the
small-generator regime $a \ll 1$. \emph{Linearisation error}
(controlled by (A3)): replacing the nonlinear block map by its
first-order Taylor expansion at the reference --- a Taylor remainder
bounded on the neighbourhood of the reference where (A3) holds. The
theorem licenses the name: without it, \eqref{eq:linear-block} is just
a definition (a convenient way to write the layer map); with it,
\eqref{eq:linear-block} is also a discretisation (a forward-Euler
approximation of an underlying continuous flow, with quantified
error). The two senses are compatible, and the latter is what makes
the linearised equation analytically tractable.\footnote{Mean-field
continuous attention flows~\citep{geshkovski2023clusters} use a
different discretisation convention: their projected dynamics
$\dot x_i = \mathrm{proj}_{x_i}(\text{attention average})$ has no
separate residual term --- the tangent projection back to the
constraint manifold plays the role of ``replace and renormalise,
continuously.'' The discrete Transformer, by contrast, \emph{adds} the
attention output (Pre-LN residual) and renormalises at the next layer.
The two conventions produce different linearisations: the mean-field
linearised sector map is $A_j - R_j$, where $R_j$ is a relaxation term
from the projection back to the manifold; the discrete Transformer's
is $I + A_j$, where $I$ is the linearisation of the residual addition.
These are two discretisations of the same underlying idea, and the
structural difference --- $R_j$ (projection-induced) versus $I$
(residual-addition-induced) --- is the difference between ``continuous
relaxation to the constraint surface'' and ``additive update with no
projection.''}

\textbf{Why eigenvalues are not enough.} In a non-autonomous depth
flow the growth of a state is a property of the \emph{product}
$\Phi(\ell, 0)$, not of any single factor $A(\ell)$. Two facts make
eigenvalues of the individual $A(\ell)$ the wrong object to consult.
(i) The product's norm can vastly exceed what the per-layer spectral
radii suggest, $\|\Phi(\ell,0)\| \gg \prod_r \rho\bigl(I + A(r)\bigr)$,
when the per-layer matrices are non-normal --- the classical
phenomenon of \emph{transient growth} in non-normal systems. (ii) The
product's growth rate is determined by alignment between successive
layers, not by the spectral radius of individual layers. The dichotomy
spectrum is defined directly on the cocycle: it asks, for each growth
scale $\gamma > 0$, whether the rescaled cocycle
$\gamma^{-(\ell-s)}\Phi(\ell,s)$ admits a clean split into growing and
decaying subspaces.

\textbf{Exponential dichotomy.} Say that $\Phi$ admits an
\emph{exponential dichotomy} with constant $C \ge 1$, rates
$0 < r_- < 1 < r_+$, and complementary projector family $P_\pm(\ell)$
(i.e., $P_+(\ell) + P_-(\ell) = I$, $P_+(\ell)P_-(\ell) = 0$)
satisfying the cocycle invariance property
\[
P_\pm(\ell)\,\Phi(\ell, s) \;=\; \Phi(\ell, s)\,P_\pm(s)
\quad\text{for all } \ell \ge s,
\]
if the two bundles grow / decay in opposite directions:
\[
\|\Phi(\ell, s)\,P_-(s)\| \;\le\; C\,r_-^{\,\ell-s} \quad (\ell \ge s),
\qquad
\|\Phi(\ell, s)\,P_+(s)\| \;\le\; C\,r_+^{\,\ell-s} \quad (\ell \le s).
\]
The first bound says the $P_-$ bundle \emph{decays forward} in time
(rate $r_-$); the second says the $P_+$ bundle \emph{decays backward}
in time, equivalently \emph{grows forward} (rate
$r_+$)~\citep{coppel1978dichotomies}. The dichotomy is the
non-autonomous analogue of the stable/unstable subspace decomposition
of an autonomous matrix. The \emph{dichotomy spectrum} of the cocycle
is the set of growth scales $\gamma > 0$ at which the rescaled cocycle
$\gamma^{-(\ell-s)}\Phi(\ell,s)$ admits \emph{no} dichotomy; it is a
compact union of at most $d$ closed
intervals~\citep{sackersell1978spectral,kloeden2011nonautonomous}, and
it is the non-autonomous replacement for the set of eigenvalue
magnitudes.

\begin{theorem}[Dichotomy spectrum controls growth]\label{thm:eigenvalue}
Assume (A1) and let $\Sigma_{\mathrm{dich}} \subset (0,\infty)$ be the
dichotomy spectrum of the tangent cocycle $\Phi$. Then
$\Sigma_{\mathrm{dich}} = \bigcup_{i=1}^{n} [\gamma_i^-, \gamma_i^+]$
with $n \le d$, and to each spectral interval corresponds an invariant
bundle of solutions whose exponential growth rates lie in that
interval; solutions in bundles below $1$ decay, in bundles above $1$
grow, and the whole trajectory is \emph{hyperbolic} iff
$1 \notin \Sigma_{\mathrm{dich}}$. In the weight-tied specialisation
$A(\ell) \equiv A$ with $A$ diagonalisable, the spectrum collapses to
the point set $\{|1+\sigma_i|\}$ and the bundles to the eigenspaces.
\end{theorem}
\begin{proof}
See \Cref{app:partI} (\Cref{pf:eigenvalue}).
\end{proof}

\textbf{One consequence worth recording.} The dichotomy spectrum is
defined directly on the discrete product $\Phi(\ell, s)$ --- there is
no separate continuous convention to convert between. Earlier
formulations in the layers-as-dynamics tradition had to reconcile two
readings of ``growth'' --- the continuous convention (positive real
part, via $e^{\sigma\ell}$) against the discrete one (magnitude
exceeding $1$, via $(1+\sigma)^\ell$), which disagree off the real
axis --- by a step-size regularity condition. With the cocycle-native
definition the ambiguity does not arise: growth means the dichotomy
spectrum lies above $1$, decay means it lies below $1$, and
hyperbolicity is the question of whether $1$ is in the spectrum. This
is the precise sense in which the dichotomy spectrum is the
non-autonomous replacement for the eigenvalue spectrum.

\textbf{Hyperbolicity --- and where the split can possibly come from.}
The spectral hypothesis under which the rest of the paper operates is
the non-autonomous saddle:

\begin{itemize}
\item[\textbf{(A4)}] \emph{The tangent cocycle $\Phi$ admits an exponential
dichotomy with constant $C$, rates $r_- < 1 < r_+$, and bundles
$V_\perp(\ell) = \operatorname{ran} P_-(\ell)$ (the \emph{decaying}
bundle) and $V_\parallel(\ell) = \operatorname{ran} P_+(\ell)$ (the
\emph{persistent} bundle), with margins $m_- := 1 - r_-$ and
$m_+ := r_+ - 1$ both positive.}
\end{itemize}

We call solutions in $V_\parallel$ the \emph{persistent} modes and
deliberately do not call them slow: slowness is a property of the
block-clock variables of \S\ref{sec:settings}, and conflating
``survives depth'' with ``evolves slowly'' is exactly the ambiguity
the two-clock construction is built to remove. The constant $C$
deserves respect: it absorbs both the non-normality of the layer maps
and the incoherence of the splitting across depth, playing the role
the eigenbasis condition number played in the autonomous picture;
every estimate below carries it explicitly, and all asymptotic
statements govern the regime $\ell \gtrsim \log C/\!\log(1/r_-)$ in
which asymptotics dominate transients.

The motivation of (A4) is representational: if the persistent bundle
were trivial, the depth flow would contract everything, all initial
conditions would collapse onto the reference trajectory, and the
output head would read a constant --- a trained network that
transports \emph{anything} to its output operates with a nontrivial
$V_\parallel$ in the regime our linearisation describes. But the
framework can say something sharper than motivation: it can say
\emph{where} the split must live, and the statement is now
layerwise.\footnote{Take the isotropic-value idealisation
$V(\ell) = c_\ell I$ that the mean-field literature analyses. Then
every sector factor degenerates to $I + M(\ell) + c_\ell\lambda_j I$:
at each layer, attention contributes a homothety --- it shifts that
layer's factor by a scalar but \emph{splits nothing} --- and with the
MLP removed ($M(\ell) \equiv 0$) the sector cocycles are products of
scalar multiples of the identity, whose dichotomy spectrum is a single
point: no saddle, at any depth profile $c_\ell$, any temperature, any
$W_Q^{(\ell)}, W_K^{(\ell)}$. This is consistent with --- indeed it is
the linearised, non-autonomous shadow of --- the mean-field
literature's own headline: attention-only flows cluster universally,
and the feed-forward layer is precisely the term their methods do not
yet cover~\citep{geshkovski2023clusters}.} The conclusion for us is
structural: \emph{whatever splits the dichotomy spectrum --- whatever
selects which feature directions persist --- is the feature-space
machinery $M(\ell)$ (and the anisotropy of $V(\ell)$), never the
mixing itself.}

\textbf{A concrete instance: stable looped Transformers.} The
dichotomy vocabulary is not only descriptive; half of it can be
\emph{enforced}. In stable looped
architectures~\citep{geiping2025scaling,prairie2026parcae}, the
recurrent update is parameterised in discrete-time linear time invariant (LTI) form,
$h_{t+1} = \bar A\,h_t + \bar B\,e + \overline{\mathcal R}(h_t, e)$,
with $\bar A = \exp(-\exp(\cdot))$ diagonal --- so that, in our
convention, the layer factor $I + A$ equals $\bar A$ and
$\mathrm{spec}(A) \subset (-1, 0)$ for \emph{any} value of the learned
parameters. For the LTI skeleton (the linearisation that drops
$\overline{\mathcal R}$) this makes the \emph{stable half} of (A4)
structural rather than hypothetical: $r_- = \rho(\bar A) < 1$ by
parameterisation, $C = 1$ (diagonal, weight-tied --- no non-normality,
no depth-incoherence, no transient regime), and the alignment
condition (A7) of \S\ref{sec:coupling:ident} holds exactly, since the
dichotomy projectors are coordinate projections and the injection gain
is diagonal. What the parameterisation deliberately \emph{forbids} is
the other half: every direction is contracting, the persistent bundle
is trivial ($V_\parallel = \{0\}$), and the distinguished trajectory
$h^\star = (I - \bar A)^{-1}(\bar B\,e + \bar\Delta)$ --- the Green
operator of \S\ref{sec:fixed} reduced to its forward branch --- is a
global attractor. All persistent content is therefore carried by the
\emph{forcing}: the looped architecture is the extreme point of this
paper's design axis at which the second coupling carries everything
and the fast cocycle is built to forget. The standard Transformer
($\varrho = 0$) is the opposite extreme, where persistence must live
in the cocycle's own $V_\parallel$. Two empirical observations
~\citep{prairie2026parcae} read naturally in this vocabulary:
training diverges exactly when the learned spectral radius crosses the
dichotomy boundary ($\rho(\bar A) \to 1^{+}$) --- the margin $m_-$ is
the operative stability frontier of real training, not a modelling
convenience --- and test-time quality improves with loop count along a
saturating exponential, the visible shape of the contraction
$\|h_t - h^\star\| \lesssim r_-^{\,t}$ toward the distinguished
trajectory. The saddle half of (A4) --- a nontrivial persistent bundle
inside the cocycle --- remains, for trained standard stacks, the point
where the network enters the analysis as an empirical object.\footnote{We
record the purely stable regime as a citable specialisation:

\textbf{(A4$^\circ$)} \emph{(Purely stable specialisation.) The
tangent cocycle admits an exponential dichotomy with
$V_\parallel = \{0\}$: $\Sigma_{\mathrm{dich}} \subset (0,1)$, the
Green operator of \S\ref{sec:fixed} reduces to its forward branch with
$\|\mathcal G\| \le C/m_-$, the distinguished trajectory is a global
attractor, and all persistence is carried by the forcing. Every
statement of \S\ref{sec:fixed} holds with the $m_+$-terms deleted.}}

The saddle also fixes the right reading of the word ``reference'':
the homogeneous trajectory is a solution, not an attractor --- states
approach it along $V_\perp(\ell)$ and leave along $V_\parallel(\ell)$,
and it is exactly the $V_\parallel$ departure that carries
token-specific content to the output layer. ``Saturation'' names
proximity to the reference trajectory, not a claim of global
convergence to it.

\section{General setting: the second coupling}\label{sec:settings}

\textbf{Two clocks, two couplings.} The forward pass carries a
\emph{depth clock}: each layer is one step of the fast dynamics,
advancing $\ell$ by $\Delta\ell = 1$ across the stack. The
architectures this paper is about add a second, coarser clock, and the
cleanest way to introduce it is through the structure of the token
axis.

\begin{itemize}
\item[\textbf{(C1)}] \emph{(Fast coupling: block-local attention.) Partition
the $T$ positions into $K = \lceil T/P\rceil$ contiguous blocks of
length $P$. Within a forward pass, causal attention is confined to the
current block: token $t$ in block $k$ attends to positions $s \le t$
in block $k$.}
\end{itemize}

The fast coupling is \emph{block-local}. This is how hierarchical
designs bound the cost of the fast path, and it makes the fast
coupling's reach finite \emph{by construction}: information cannot
cross a block boundary through attention at all. Whatever crosses
blocks must go through the second coupling:

\begin{itemize}
\item[\textbf{(C2)}] \emph{(Slow coupling: compression and gated injection.)
A causal compression operator $\Pi$ maps completed blocks to per-block
summaries $y_k = \Pi(x_{(k-1)P+1}, \ldots, x_{kP})$; a gate injects
$\varrho\,\psi(y_{k-1})$ --- the same vector for every token of block
$k$ --- additively into the fast stream. Here $\varrho \in \R$ is a
scalar amplitude and $\psi : \R^d \to \R^d$ is a $C^1$ injection map;
at $\varrho = 0$ the model is exactly a block-local standard
Transformer.}
\end{itemize}

We refer to this model class --- block-local fast coupling,
compression-and-gate slow coupling --- as (C1)--(C2); the full-context
causal Transformer is its leaky variant, in which attention places
$O(\varepsilon_{\mathrm{leak}})$ mass outside the current block, and
we return to it in \S\ref{sec:discussion}. The held-summary injection
has a live architectural instance: the input injection $\bar B\,e$ of
stable looped Transformers~\citep{prairie2026parcae}, where the
summary is a learned encoder output held for the entire pass ---
$L_r = \infty$ in the notation introduced below.

At the token level, the injection is added \emph{alongside} the
renormaliser:
\begin{equation}\label{eq:fast_update}
  x^{(\ell)}_{t} \;=\; x_t^{(\ell-1)} + \mathcal{N}^{(\ell-1)}_t\bigl(x^{(\ell-1)}\bigr)
  \;+\; \varrho\,\psi\bigl(\hat y^{(\ell-1)}_{k(t)-1}\bigr),
  \qquad k(t) = \lceil t/P \rceil,
\end{equation}
where $\hat y_{k-1}$ is the \emph{held} summary of the most recent
completed block.\footnote{We write $\hat y_{k-1}$ for the held
(sample-and-hold) value of the previous block's compressed summary ---
the hat denotes \emph{holding}, not statistical estimation. Throughout
the paper, $\hat y$ denotes a value held constant until the next
refresh, i.e.\ across $L_r$ consecutive layers.} Writing the injection
as an added term keeps the forced linearisation affine in $\psi$,
which \S\ref{sec:fixed} builds on; the alternative --- injecting into
$\mathcal S$'s argument, which would force the slow signal onto
$\Mcalnorm$ and let the softmax gate it nonlinearly --- is an
architectural variant outside the present scope.

\textbf{The timescale ratio.} The held state $\hat y$ is refreshed
once every $L_r$ layers, and we call
$\varepsilon := 1/L_r \in (0,1]$ the \emph{timescale ratio} of the
architecture. Two design parameters, two different jobs: the block
length $P$ controls \emph{spatial} compression, the refresh period
$L_r$ controls \emph{temporal} cadence. They are logically
independent, and the canonical hierarchical choice couples them,
$L_r = P$ --- one slow tick per block of new tokens, in a streaming
reading --- giving $\varepsilon = 1/P$. The theory below needs only
that $\varepsilon$ is small. In continuous depth-time, idealising the
sample-and-hold refresh as a slow relaxation,\footnote{The
architecture refreshes $\hat y$ by assignment every $L_r$ layers. The
relaxation ODE $\frac{d\hat y}{d\ell} = \varepsilon(\Pi(x) - \hat y)$
is the standard continuous idealisation of this sampled-data loop ---
on the slow timescale $\tau = \varepsilon\ell$, the discrete refresh
and the continuous relaxation agree to first order in $\varepsilon$,
and the $O(\varepsilon^2)$ discrepancy is absorbed into the reduction
remainders of \S\ref{sec:fixed}.} the closed-loop system in the
linearised regime is
\begin{equation}\label{eq:ode_full}
  \frac{dx}{d\ell} \;=\; A(\ell)\,x + \varrho\,\psi(\hat y),
  \qquad
  \frac{d\hat y}{d\ell} \;=\; \varepsilon\, g_\Pi(x, \hat y).
\end{equation}
The change of variables $\tau := \varepsilon\ell$ is the standard
singular-perturbation move: it stretches the slow variable's clock so
that, in the new time $\tau$, $\hat y$ evolves at $O(1)$ while $x$'s
evolution carries an explicit $\varepsilon$ in front of its
derivative. This is the canonical Tikhonov form for fast--slow
systems, and the form on which invariant-manifold theory operates
directly. Concretely, the system becomes
\begin{equation}\label{eq:ode_slowtime}
  \varepsilon\,\frac{dx}{d\tau} \;=\; A(\tau/\varepsilon)\,x + \varrho\,\psi(\hat y),
  \qquad
  \frac{d\hat y}{d\tau} \;=\; g_\Pi(x, \hat y),
\end{equation}
with $x$ fast and $\hat y$
slow~\citep{fenichel1979geometric,hairer1996solving}. The first
equation has $\varepsilon$ multiplying the derivative of the fast
variable --- the algebraic signature of fast dynamics slaved to a slow
clock. The second has an $O(1)$ derivative --- the slow variable's
free evolution. Equivalently: per slow tick $\Delta\tau = 1$, the fast
variable takes $\Delta\tau/\varepsilon = L_r$ layers to catch up, and
the slow variable makes one $\Pi$-refresh. All subsequent analysis
operates on this slow-time form.

\textbf{Which directions are slow, and why the reduction is
legitimate.} We reduce the full depth flow on $\R^{T\times d}$
(dimension $Td$) to a slow flow on the blockwise-homogeneous manifold
$\mathcal H_P \subset \R^{T\times d}$ (dimension $Kd$) by projecting
out the \emph{transverse} directions --- the within-block
token-inequalities.\footnote{Geometrically, the slow manifold
$\mathcal H_P$ is a submanifold of the full state space. Directions
along it are \emph{tangent} (perturbations that stay
within-block-homogeneous, shifting each block's common value
independently); directions pointing away from it are \emph{transverse}
(perturbations that break within-block homogeneity).} The reduction of
\eqref{eq:ode_slowtime} onto a slow manifold requires the transverse
fast dynamics to contract, and the sector factorisation says exactly
which directions those are --- uniformly in depth.\footnote{The sector
factorisation of \Cref{thm:factor} identifies the tangent directions
with sector $1$ and the transverse directions with sectors $j \ge 2$
--- so ``tangent'' and ``transverse'' in the language of invariant
manifolds coincide with ``sector $1$'' and ``sector $j\ge2$'' in the
language of \Cref{thm:factor}.} Consider the blockwise-homogeneous
manifold
\[
\mathcal H_P \;:=\; \{X : x_t = x_s \text{ whenever } t, s
\text{ lie in the same block}\},
\]
which is exactly invariant under (C1)--(C2) at every layer:
block-local attention preserves within-block homogeneity by
\Cref{lem:homog}, and the injection is constant within each block.
Transverse to $\mathcal H_P$ sit the within-block fluctuation sectors
of \Cref{thm:factor}, with sector cocycles $\Phi_j$, $\lambda_j < 1$
--- and the transverse hypothesis is that they contract, uniformly and
with room to spare over the tangent growth:

\begin{itemize}
\item[\textbf{(A5)}] \emph{Each fluctuation cocycle satisfies
$\|\Phi_j(\ell,s)\| \le C_f\, r_f^{\,\ell-s}$ for all $j \ge 2$, with
$C_f \ge 1$ and $r_f < 1$, and the transverse contraction dominates
the tangent growth with a uniform gap: $r_f\, r_+ \le \theta$ for some
$\theta < 1$.}
\end{itemize}

The first condition says the fluctuations contract; the second says
they contract faster than the persistent tangent directions grow ---
the normal-hyperbolicity gap that the reduction needs.

Under (A5), $\mathcal H_P$ is a normally hyperbolic invariant manifold
of the non-autonomous flow with \emph{attracting} transverse
directions, and the invariant-manifold theory of non-autonomous
systems~\citep{kloeden2011nonautonomous,fenichel1979geometric} applies
with no caveat about saddles.\footnote{\label{fn:fenichel}In the
autonomous case this is Fenichel theory: for
$\varepsilon\dot x = f(x,y)$, $\dot y = g(x,y)$ with normally
hyperbolic critical manifold, coordinates $(a,b,v)$ ---
transverse-stable, transverse-unstable, slow --- reduce the flow to
$\dot a = \Lambda a$, $\dot b = \Gamma b$,
$\dot v = \varepsilon h(v,\varepsilon) + O(|a||b|)$, the slow manifold
$\{a=b=0\}$ persisting for small $\varepsilon$. Non-autonomously the
same programme runs with the dichotomy projectors in place of spectral
projectors and pullback attraction in place of
attraction~\citep{kloeden2011nonautonomous}; under (A5) the
transverse-unstable block $b$ is absent, which is why the reduction
holds for generic initial data.}

Geometric singular perturbation theory then applies directly: there
exists a slow manifold within $O(\varepsilon)$ of $\mathcal H_P$, on
which a reduced flow governs the long-time dynamics, and generic
initial data converges to the slow manifold at the transverse rate
$r_f$. The persistent bundle $V_\parallel(\ell)$ is \emph{tangent} to
$\mathcal H_P$, not transverse to it, so it belongs to the reduced
system and obstructs nothing. The resulting hierarchy of rates is the
paper's fast--slow structure in final form: within-block fluctuations
contract fastest (rate $r_f$, transverse --- they wash out); the
decaying bundle $V_\perp$ contracts next (rate $r_-$, tangent); the
persistent bundle $V_\parallel$ grows or persists (rate $r_+$,
tangent) and carries the content; and the held summary $\hat y$ moves
slowest of all, at $O(\varepsilon)$ per layer, bridging blocks. Fast
means transverse and contracting; slow means tangent and persistent,
or on the block clock.

\begin{center}
\footnotesize
\begin{tabular}{llll}
\toprule
Rate & Direction & Type & Physical meaning \\
\midrule
$r_f$ & transverse, within-block & fast contraction & block-internal token-inequality washed out \\
$r_-$ & tangent $V_\perp$ in $\mathcal H_P$ & medium decay & decaying mode in the slow flow \\
$r_+$ & tangent $V_\parallel$ in $\mathcal H_P$ & growth / persistent & persistent mode carrying content \\
$O(\varepsilon)$/layer & $\hat y$, block-clock & slowest & cross-block slow variable \\
\bottomrule
\end{tabular}
\end{center}

\begin{remark}[Choices of $\Pi$]\label{rem:pi}
The compression operator accommodates a range of designs, in order of
increasing aggressiveness: block-mean,
$\Pi_{\mathrm{mean}}(x)_k = P^{-1}\sum_j x_{(k-1)P+j}$ (linear, every
token contributes); stride, $\Pi_{\mathrm{stride}}(x)_k = x_{kP}$
(deterministic subsampling); content-addressed sparse selection, where
a small learned controller picks representative positions per
block~\citep{csa2026,msa2026}; and multi-level composition
$\Pi = \Pi^{(J)}\circ\cdots\circ\Pi^{(1)}$, as in hierarchical
recurrent designs~\citep{HRM2025,geiping2025scaling}. In the
mean-field reading the block is an empirical measure over its tokens
and a compression is a statistic of that measure --- the block mean is
its first moment~\citep{geshkovski2023clusters} --- which is the
geometric face of the statistical fact, proved in
\Cref{prop:suffstat}, that the block mean is the \emph{sufficient}
compression for exponential-family block structure. On $\mathcal H_P$
all within-block tokens agree and every reasonable $\Pi$ returns the
common value, so the reduced theory is $\Pi$-agnostic; off the
manifold, results that use the smooth geometry of the slow manifold
need a $C^1$ compression (block-mean qualifies; hard selectors do
not, cf.\ \S\ref{sec:intro:renorm}).
\end{remark}

\section{Couplings: Architectural Parameters as Theory-Space Coordinates}\label{sec:coupling}

\textbf{Why two couplings?} Under (C1) the fast coupling's receptive
field is finite by construction: attention cannot cross a block
boundary, and the second coupling is the only cross-block channel. Its
load-bearing role is conditional twice over --- on the
\emph{architecture} side it matters once the fast path has exhausted
within-block mixing (the depth-saturation regime, quantified by
\Cref{thm:duality}), and on the \emph{data} side it matters only if
consecutive blocks share structure that block-local prediction cannot
recover (\S\ref{sec:fixed:load}).

A \emph{coupling}, in this framework, is anything that fixes one of
the maps appearing in the fast--slow system \eqref{eq:ode_full}: the
timescale ratio $\varepsilon$, the compression $\Pi$, the norm scaling
$\mathcal S$, the gate amplitude $\varrho$, and the injection $\psi$
are the \emph{framework couplings} --- interpretable architectural
hyperparameters with direct dynamical meaning, fixed before training
--- while the trained weights $W$, including $W_Q^{(\ell)}$,
$W_K^{(\ell)}$, $W_V^{(\ell)}$; $W_1^{(\ell)}$, $b_1^{(\ell)}$,
$W_2^{(\ell)}$, $b_2^{(\ell)}$ in \eqref{eq:layer_evolution}; the
gains of $\mathcal S$; and, if learned, the parameters of $\psi$, are
the training-time variable that \emph{realises} the generator family.
The realisation goes only through the combinations \eqref{eq:MV-def}:
$A(\ell) = M(\ell) + V(\ell)$.

Two consequences are worth recording before either is used. First,
many $W$ realise one generator family: the specific values of
$W_1^{(\ell)}, W_2^{(\ell)}$ are not uniquely determined by
$\{A(\ell)\}$.\footnote{For example, for elementwise $\phi$ any
permutation matrix $G$ (acting with $W_1 \mapsto GW_1$,
$b_1 \mapsto Gb_1$, $W_2 \mapsto W_2G^{-1}$) leaves $M(\ell)$
unchanged; for positively homogeneous $\phi$ (the ReLU family) the
same holds for any positive diagonal $G$. Similar redundancies exist
for the other weight components; the map $W \mapsto \{A(\ell)\}$ is
many-to-one.} This means the asymptotic observables (cocycle,
dichotomy spectrum, distinguished trajectory) depend only on the
effective degrees of freedom $\{A(\ell)\}$, not on the specific $W$
realising them --- and is exactly why we keep $W$ (the raw
training-time coordinate) and the cocycle (the asymptotic observable)
as separate objects in the framework.

Second, $W_Q^{(\ell)}$ and $W_K^{(\ell)}$ enter no $A(\ell)$ at all:
by \Cref{thm:factor}(1) the kernel along the reference is universal at
every layer, so the query--key weights touch the linearised dynamics
only through the fluctuation sectors' state dependence, at second
order in the inhomogeneity. Both consequences return, sharpened into
identification statements, in \S\ref{sec:coupling:ident}.

\subsection{Depth $\leftrightarrow$ coarse-graining: the duality}\label{sec:coupling:duality}

Depth $\ell$ and the coarse-grain index $k$ are \emph{dual} in the
fast--slow sense: depth advances the fast clock, coarse-graining
defines the slow variable, and the two are tied through
$\varepsilon$. The precise statement has three parts --- on the
invariant manifold coarse-graining loses nothing; off it, the
discrepancy decays at a computable rate; and beyond the saturation
depth, running the layers \emph{is} running the coarse variable.

\begin{theorem}[Depth--coarse-grain duality]\label{thm:duality}
Assume (A1)--(A5), the model class (C1)--(C2), and block-mean $\Pi$.
Let $\Phi^{\mathrm{full}}_\ell$ denote the linearised depth flow of
\eqref{eq:ode_full} over layers $0,\ldots,\ell$ at fixed $\hat y$.
Then:
\begin{enumerate}
\item \emph{(Exactness on the manifold.)} $\mathcal H_P$ is invariant
  under $\Phi^{\mathrm{full}}_\ell$ for every $\ell$; on it, $\Pi$ is
  a bijection onto the block coordinates $(y_1, \ldots, y_K)$, and the
  depth flow acts diagonally --- each $y_k$ evolves under
  $\dot y_k = A(\ell)\,y_k + \varrho\,\psi(\hat y_{k-1})$. On the
  manifold, depth dynamics and coarse-grained dynamics are the same
  object.
\item \emph{(Decay off the manifold.)} For arbitrary $\delta X$,
  decomposed as $\delta X^{\mathcal H} + \delta X^{\perp}$ along
  $\mathcal H_P$ and the fluctuation sectors,
  \[
  \bigl\|\Phi^{\mathrm{full}}_\ell(\delta X) - \Phi^{\mathrm{full}}_\ell(\delta X^{\mathcal H})\bigr\|
  \;\le\; C'\,\kappa_\alpha\,C_f\, r_f^{\,\ell}\,\|\delta X^{\perp}\|,
  \]
  with $\kappa_\alpha$ the (fixed, depth-independent) eigenbasis
  condition number of $\alpha^\star$ and $C'$ an absolute constant.
\item \emph{(Duality.)} For any tolerance
  $\epsilon_{\mathrm{tol}} > 0$ and
  $\ell \ge \tau_{\mathrm{sat}}(\epsilon_{\mathrm{tol}}) :=
  \log(C'\kappa_\alpha C_f/\epsilon_{\mathrm{tol}})/\log(1/r_f)$,
  \begin{equation}\label{eq:dual_formula}
  \Phi^{\mathrm{full}}_\ell \;=\; \iota \circ \Phi(\ell,0) \circ \tilde\Pi \;+\; O(\epsilon_{\mathrm{tol}}),
  \end{equation}
  where $\iota$ injects block states back into $\mathcal H_P$,
  $\tilde\Pi$ is the sector-$1$ spectral compression and is the tangent cocycle
  of \Cref{thm:factor}:\footnote{The sector-$1$ spectral compression $\tilde\Pi$ 
  is equal to block-mean $\Pi$ on $\mathcal H_P$, and everywhere in the doubly
  stochastic case (\Cref{rem:sink}). For the causal kernel, the
  block-mean form of the factorisation holds through an intermediate
  depth, with error $\theta^{\ell/2}$ under the gap condition of (A5)
  (\Cref{rem:duality-blockmean}).} past the saturation depth, the depth flow
  \emph{factors through} the coarse-grained variable. Moreover the
  tangent cocycle over one slow
  tick, $\Phi(\ell+L_r,\ell)$, preserves the dichotomy bundles
  $V_\parallel, V_\perp$, so one slow-clock step acts on the coarse
  variable with the same persistent/decaying structure as $L_r$ fast
  steps.
\end{enumerate}
\end{theorem}
\begin{proof}
See \Cref{app:partII} (\Cref{pf:duality}); part 1 combines
\Cref{lem:homog} with block-locality and the within-block constancy of
the injection, part 2 is \Cref{thm:factor} plus (A5), and part 3
composes the two.
\end{proof}

The composition $\iota \circ \Phi(\ell, 0) \circ \Pi$ is a
\emph{three-step pipeline}: $\Pi$ compresses the full state
$\R^{T\times d}$ to the block coordinates
$(y_1, \ldots, y_K) \in \R^{K\times d}$; $\Phi(\ell, 0)$ advances each
block independently under the tangent cocycle (per-block diagonal
action --- \Cref{thm:factor}(2) applied blockwise); $\iota$ re-injects
the block coordinates back into the blockwise-homogeneous form
$\mathcal H_P$. The formula \eqref{eq:dual_formula} says that this
three-step pipeline agrees with the actual depth flow
$\Phi^{\mathrm{full}}_\ell$ on the full state to within
$O(\epsilon_{\mathrm{tol}})$. Equivalently: past the saturation depth,
\emph{depth dynamics on $\R^{T\times d}$ and block-cocycle dynamics on
$\R^{K\times d}$ are two equivalent descriptions of the same
evolution}, mediated by compression $\Pi$ and re-injection $\iota$.
The dimensional reduction $T \to K$ is the practical gain: the
dichotomy analysis of \S\ref{sec:fixed} operates on the $K$-block
tangent cocycle, not on the full $T$-token flow.

The duality table summarises the correspondence; each row is an
instance of ``depth past saturation = coarse-graining'':

\begin{center}
\footnotesize
\begin{tabular}{lll}
\toprule
Depth $\ell$ (fast clock) & $\leftrightarrow$ & Coarse index $k$ (slow clock) \\
\midrule
One Euler step of $\dot x = A(\ell)x + \varrho\psi(\hat y)$ & $\leftrightarrow$ & One refresh $\hat y \leftarrow \Pi(x)$ per $L_r$ layers \\
Within-block mixing (sectors $\lambda_j < 1$) & $\leftrightarrow$ & Fluctuations averaged out by $\Pi$ \\
Depth saturation ($\ell \ge \tau_{\mathrm{sat}}$) & $\leftrightarrow$ & Exactness of coarse-graining \\
Persistent bundle $V_\parallel(\ell)$ & $\leftrightarrow$ & Content carried across blocks by $\hat y$ \\
Step $\Delta\ell = 1$ & $\leftrightarrow$ & Step $\Delta\tau = \varepsilon^{-1}\Delta\ell$ \\
\bottomrule
\end{tabular}
\end{center}

\Cref{thm:duality} is an \emph{architecture-level} statement: past
saturation, the depth flow on the full state factors through the block
summary, regardless of what the data looks like. This is what makes
the slow path \emph{available} --- the architecture \emph{can} route
information across blocks through the gate. But availability is not
utility: \emph{whether} this routed information is informative about
the next block depends on whether the data carries cross-block latent
structure. The duality quantifies the architecture's reach; the
data-side analysis quantifies whether that reach is on a worthwhile
target. Both are required for the full picture --- the
architecture-side result is the input to the data-side discussion, and
the two are deliberately separated.

\subsection{The training flow on $W$-space}\label{sec:coupling:flow}

The weights $W$ live on the depth axis: each layer applies
$W^{(\ell)}$, advancing one Euler step. Training is a flow on
$W$-space: each step moves $W$ along $-\nabla_W\Lcal$, where $\Lcal$
is the full-batch training loss on a fixed data distribution --- the
standard gradient-flow idealisation of full-batch gradient descent.
The flow is autonomous in $W$; stochastic minibatches, adaptive
optimisers and schedules break exact autonomy, and we work in the
idealisation and flag it as such. The trained $W^\star$ realises,
through \eqref{eq:MV-def}, the generator family and hence the cocycle:
$W$ is the raw coordinate whose training-time convergence learns the
dynamics, the cocycle the asymptotic observable that the
identification below reads.

The framework uses two flows that we will keep carefully separate.
\emph{Training flow} (in parameter space): the gradient flow
$\dot W = -\nabla_W \Lcal(W)$ converges to a trained $W^\star$.
\emph{Depth flow} (in token state space, after linearisation):
$\delta X^{(\ell)} = \Phi^{\mathrm{full}}_\ell\,\delta X^{(0)}$, with
growth governed by the dichotomy spectrum. The two flows are linked by
a chain:
\[
W^\star \;\xrightarrow{\;W \mapsto \{A(\ell)\}\;}\; \{A(\ell)\} \;
\xrightarrow{\;\Phi(\ell,s) = \prod(I+A(r))\;}\; \Phi(\ell, s) \;
\xrightarrow{\;\text{read by \S\ref{sec:coupling:ident}}\;}\; \text{asymptotic observables}.
\]
Training produces $W^\star$; $W^\star$ realises a generator family;
the generator family determines the cocycle; the identification
theorems below read the cocycle. The vocabulary of fast-mode and
slow-mode perturbations is built to express \emph{which
$W$-directions} in the first flow's state space move \emph{which
asymptotic observables} in the second flow's analysis. The dichotomy
bundle decomposition we introduce next is the bridge between the two
flows.

A perturbation $\delta W$ induces a family
$\delta A(\ell) := \nabla_W A(\ell)(W^\star)[\delta W]$, and the
dichotomy projectors of (A4) decompose each member
as\footnote{The dichotomy projectors $P_\pm(\ell)$ split
$\R^d = V_\parallel(\ell) \oplus V_\perp(\ell)$ at each depth, so any
matrix $\delta A(\ell)$ has a four-block decomposition
\[
\delta A(\ell) = \begin{pmatrix} P_+ \delta A(\ell) P_+ & P_+ \delta A(\ell) P_- \\ P_- \delta A(\ell) P_+ & P_- \delta A(\ell) P_- \end{pmatrix}.
\]
We group these as a \emph{persistent block} ($P_+ \to P_+$), a
\emph{decaying block} ($P_- \to P_-$), and two \emph{skew blocks}. The
dichotomy invariance $P_\pm(\ell)\Phi(\ell,s) = \Phi(\ell,s)P_\pm(s)$
means the persistent block perturbs the dynamics inside $V_\parallel$
(rate $r_+$), the decaying block perturbs $V_\perp$ (rate $r_-$), and
the skew blocks measure cross-bundle leakage.}
\[
\delta A(\ell) \;=\; \underbrace{P_+(\ell{+}1)\,\delta A(\ell)\,P_+(\ell)}_{\text{persistent block}}
\;+\; \underbrace{P_+\,\delta A\,P_- + P_-\,\delta A\,P_+}_{\text{skew blocks}}
\;+\; \underbrace{P_-(\ell{+}1)\,\delta A(\ell)\,P_-(\ell)}_{\text{decaying block}},
\]
and we call $\delta W$ a \emph{fast-mode perturbation} if
$\delta A(\ell) = P_-(\ell+1)\,\delta A(\ell)$ for every $\ell$ (its
range lies in the decaying bundle) and a \emph{slow-mode perturbation}
if $\delta A(\ell) = P_+(\ell+1)\,\delta A(\ell)\,P_+(\ell)$.

The hypothesis this vocabulary is built for is the following.

\begin{itemize}
\item[\textbf{(A6)}] \emph{Training converges to a $W^\star$ whose late-phase
loss curvature lies in the slow-mode directions; the fast-mode
directions are exhausted early.}
\end{itemize}

(A6) is the bridge between the two flows: the identification theorems
below state that fast-mode perturbations are invisible to the
asymptotic observables, and (A6) states that the trained network's
residual curvature lives exactly in the visible directions --- the
slow-mode directions carry both the observables' sensitivity and the
optimiser's remaining signal, while the invisible (fast-mode)
directions have been spent. Its empirical shadow is the two-phase
timeline of deep learning --- local, abundant patterns first,
longer-range structure
later~\citep{rahaman2019spectral,kalimeris2019sgd,olsson2022induction}
--- and the slow path is the submodule whose capacity is aimed at the
second phase. (A6) is a hypothesis about what training converges to,
not a consequence of loss minimisation; it is testable exactly insofar
as the asymptotic observables are measurable, which is what the next
subsection establishes.

\subsubsection{Identification: what the structured limit can and cannot see}\label{sec:coupling:ident}

The statistical question underneath the architecture is an
identification question: which directions in parameter space move the
observables of the structured limit, and which are quotiented out?
Two observables are available at the distinguished trajectory: the
\emph{trajectory displacement} (where the forced bounded solution
$u^\star_\varrho(\ell)$ sits relative to the unforced reference
$u \equiv 0$; the geometric channel, computed in \S\ref{sec:fixed}),
and the \emph{transverse decay profile}, which we call the
\emph{depth-saturation channel}: the rates $r_f, r_-$ together
characterise when and how quickly the fast path has nothing left to
move, and hence how the depth flow approaches its asymptotic regime.

\begin{itemize}
\item[\textbf{(A7)}] \emph{The injection respects the splitting: the
linearisation $D\psi$ evaluated along the distinguished trajectory
$u^\star_\varrho$ commutes with the dichotomy projectors
$P_\pm(\ell)$, i.e.\
$D\psi(u^\star_\varrho(\ell))\,P_\pm(\ell) =
P_\pm(\ell)\,D\psi(u^\star_\varrho(\ell))$.}
\end{itemize}

Geometrically, (A7) says the linearised feedback $\varrho\,D\psi$
preserves the persistent and decaying bundles --- it does not rotate
the dichotomy's bundle structure. It simplifies the proof by ensuring
the perturbation respects the splitting that (A4) already provides,
and it holds in particular when $D\psi$ is diagonal in the dichotomy's
basis, or more generally when $\psi$ is a coordinate-wise map whose
coordinates align with $V_\parallel, V_\perp$ --- exactly the case of
the diagonal injection gains of stable looped architectures
(\S\ref{sec:intro:ode}).

\begin{theorem}[Fast-mode invariance of the trajectory displacement]\label{thm:fastmode}
Assume (A1)--(A5) and the alignment condition (A7). Let
$u^\star_\varrho(\cdot\,;W)$ denote the unique bounded solution of the
forced linear system
$u(\ell+1) = (I + A(\ell))\,u(\ell) + \varrho\,\psi(\hat y)$ at fixed
$\hat y$ (the distinguished trajectory of \Cref{thm:fixed:eq}), and
let $\delta W$ be a fast-mode perturbation of size
$\eta := \sup_\ell \|\delta A(\ell)\|$, with $\eta$ small enough that
the dichotomy persists (by roughness,
$\eta \lesssim \min(m_-, m_+)/C^2$; \Cref{thm:fixed:stability}). Then
the persistent component of the distinguished trajectory is unchanged
to first order,
\[
P_+(\ell)\,u^\star_\varrho(\ell;\, W^\star + \delta W)
\;=\; P_+(\ell)\,u^\star_\varrho(\ell;\, W^\star) + O(\eta^2)
\quad\text{uniformly in } \ell,
\]
and so is the displacement observable
$d_+(\varrho; W) := \sup_\ell\|P_+(\ell)\bigl(u^\star_\varrho(\ell;W) -
u^\star_0(\ell;W)\bigr)\|$.
\end{theorem}
\begin{proof}
See \Cref{app:partII} (\Cref{pf:fastmode}).
\end{proof}

\emph{Why fast-mode perturbations are invisible.} The perturbation
does change the dynamics --- the generator family moves by
$\delta A(\ell)$ and the cocycle with it --- but the change is
confined where the persistent observable cannot see it.
Differentiating the Green-operator representation of the distinguished
solution (\S\ref{sec:fixed}) produces the correction
$\delta u^\star(\ell) = (\mathcal G\,[\delta A(\cdot)\,
u^\star(\cdot)])(\ell) + O(\eta^2)$, and both branches of the Green
kernel kill its persistent component: on the forward branch, the
correction injected at layer $s$ has range in $V_\perp(s{+}1)$, and
the bundle invariance
$P_\pm(\ell)\Phi(\ell,s) = \Phi(\ell,s)P_\pm(s)$ keeps it inside the
decaying bundle ever after, where $P_+(\ell)$ annihilates it; on the
backward branch, the kernel carries the factor $P_+(s)$, and
$P_+(s)\,\delta A(s) = 0$ is precisely the definition of a fast-mode
perturbation. So $P_+\,\delta u^\star \equiv 0$ at first order, while
the decaying component $P_-\,\delta u^\star$ moves freely --- the
perturbation is absorbed entirely by directions that wash out before
reaching the output. This is the structural reason the identification
analysis partitions parameter space: the persistent component of
$u^\star_\varrho$ is what the output head reads, and it cannot
register which $\delta W$ inside the decaying bundle was added.

\begin{proposition}[Universality of the frozen kernel]\label{prop:kernel}
Along any token-homogeneous reference trajectory the frozen attention
kernel is the causal-uniform kernel of \Cref{thm:factor}(1) at every
layer, independently of the layer's $W_Q, W_K, W_V$ and of the
reference state; the depth-saturation observable therefore depends on
$W$ only through $M(\ell)$ and $V(\ell)$, never through the kernel
itself.
\end{proposition}
\begin{proof}
Immediate from \Cref{thm:factor}(1)--(2).
\end{proof}

A common intuition about late layers, made precise: at saturation, the
query--key machinery has no first-order handle left --- ``who attends
to whom'' has equalised, at every layer at once, and what remains
identifiable is ``what is transported'' (the sector cocycles) and
``what the gate does'' ($\varrho, \psi$). Thus fast-mode weight
directions (\Cref{thm:fastmode}) and the query--key weights at
saturation (\Cref{prop:kernel}) lie in the \emph{unidentified}
quotient --- they were load-bearing for getting \emph{to} the
saturation regime, but leave no first-order trace on where the
distinguished trajectory runs. One further washing-out result closes
the circle: not only are fast-mode \emph{weight} perturbations
invisible, fast-mode \emph{state} content never reaches the slow path
either.

\begin{theorem}[Decaying content does not reach the injection]\label{thm:geom}
Assume (A1)--(A5), and let
$\zeta^{(\ell)} := \Pi(\mathcal S(x^{(\ell)}))$ be the pooled
normalised state feeding $\psi$, split along the dichotomy bundles as
$\zeta = P_+(\ell)\zeta + P_-(\ell)\zeta$. With $L_\psi$ the Lipschitz
constant of $\psi$ on $\Mcalnorm$,
\[
\bigl\|\psi(\zeta^{(\ell)}) - \psi(P_+(\ell)\zeta^{(\ell)})\bigr\|
\;\le\; L_\psi\,C\, r_-^{\,\ell}\,\|P_-(0)\zeta^{(0)}\|
\;+\; L_\psi\, C'\,\kappa_\alpha C_f\, r_f^{\,\ell}\,\|\delta X^{\perp,(0)}\| .
\]
\end{theorem}
\begin{proof}
See \Cref{app:partII} (\Cref{pf:geom}).
\end{proof}

\emph{What the bound says.} The injection reads the pooled state
$\zeta^{(\ell)}$; the theorem bounds the discrepancy between what the
injection actually sees and what it would see if the state were first
projected onto the persistent bundle. The first term is the (A4)
dichotomy bound: the $P_-$-component of the state decays at rate
$r_- < 1$. The second is \Cref{thm:duality}(2): the
fluctuation-sector contamination of the pool decays at rate
$r_f < 1$. Both rates are strictly below one, so for
$\ell \ge \tau_{\mathrm{sat}}$ the discrepancy is exponentially small
and $\psi$ reads only the persistent component of the state to leading
order. To leading order past the transient regime, the slow path's
input --- hence its contribution --- depends only on the persistent
component: the fast mode is the easier-to-learn information (local
patterns, abundant, no slow path needed, driving the \emph{initial}
loss drop), and the slow mode is what the architecture is
\emph{designed} to carry, exponentially insulated from everything
else.

\emph{The identification partition at a glance.} The results of this
section, together with the fixed-point computation of
\S\ref{sec:fixed}, partition parameter space and observable space into
visible and invisible parts. The table below makes the partition
explicit: each row is a parameter direction, and the cited result is
the formal visibility or invisibility statement. This is the formal
content of the word ``identifiable'' used throughout the paper: a
parameter direction is identifiable from the structured-limit
observables iff it moves the asymptotic observables at first order ---
equivalently, iff it does not lie in the invisible quotient. With the
partition in hand, we turn to \S\ref{sec:fixed} and compute the
distinguished trajectory that the slow path actually drives.

\begin{center}
\footnotesize
\begin{tabular}{p{0.28\textwidth}p{0.30\textwidth}p{0.32\textwidth}}
\toprule
Parameter direction & Observable it moves & First-order visibility \\
\midrule
Slow-mode $W$ ($P_+\delta A P_+$) & trajectory displacement $d_+$; persistent bundle & visible (generic; \Cref{prop:shift}) \\
Fast-mode $W$ ($P_-\delta A$) & --- & invisible (\Cref{thm:fastmode}) \\
$W_Q^{(\ell)}, W_K^{(\ell)}$ at saturation & frozen kernel & invisible (\Cref{prop:kernel}) \\
$M(\ell), V(\ell)$ jointly & transverse rates $r_f$, margins $m_\pm$ & visible (\Cref{thm:duality}(2)) \\
Gate $\varrho$, injection $\psi$ & trajectory displacement $d_+$ & visible (\Cref{prop:shift}) \\
Compression $\Pi$ & nothing on $\mathcal H_P$; smoothness off it & \Cref{rem:pi} \\
\bottomrule
\end{tabular}
\end{center}
\normalsize

The structured limit induces a partition of parameter space: the
identified directions are exactly the slow-mode weights and the
slow-path couplings. That is the formal sense in which the slow path's
contribution is attributable: any first-order change in the
structured-limit observables must come from the identified side, and
the identified side is where the slow path lives.

\section{The Distinguished Trajectory}\label{sec:fixed}

A non-autonomous system has no equilibria: the object that plays the
fixed point's role is the \emph{distinguished bounded solution} of the
forced linearisation --- under an exponential dichotomy it exists, it
is unique, and it is what the gate displaces. Work in deviation
coordinates $u(\ell) := y(\ell) - x^\star(\ell)$ in the homogeneous
sector, so the unforced reference is $u \equiv 0$; on $\mathcal H_P$
the pooled state coincides with the block state, so the closed loop in
one block is\footnote{We write $\psi(u)$ for
$\psi(x^\star(\ell) + u)$; the injection is evaluated at the pooled
state of the previous block, and on the slow manifold consecutive
block states agree to $O(\varepsilon)$ --- the held state tracks $u$
with a lag that perturbs nothing at first order. We use $u$ as a
deviation variable, not as a score function.}
\begin{equation}\label{eq:closedloop}
  u(\ell+1) \;=\; \bigl(I + A(\ell)\bigr)\,u(\ell) + \varrho\,\psi\bigl(u(\ell)\bigr).
\end{equation}

In an autonomous system with generator $A$, the fixed point $u^\star$
satisfying $Au^\star + \varrho\psi(u^\star) = 0$ plays a central role:
it is the asymptotic anchor, the reference point for linearisation,
the object of perturbation theory. A non-autonomous system has no
fixed point (the depth-varying $A(\ell)$ admits no constant solution
in general). The distinguished bounded trajectory
$u^\star_\varrho(\ell)$ of \eqref{eq:closedloop} plays exactly this
role in the non-autonomous setting: it is the \emph{unique bounded
solution} of the closed-loop equation, the anchor around which we
linearise, and the object whose perturbation behaviour
\Cref{thm:fixed:stability} quantifies. The word ``distinguished''
means ``the unique bounded one'': the trajectory itself evolves with
$\ell$, just as a fixed point of an autonomous system would sit still.

\subsection{The bounded solution and its displacement}\label{sec:fixed:eq}

Under the dichotomy (A4), the linear part of \eqref{eq:closedloop}
comes with a \emph{Green kernel},
\[
G(\ell, s) \;=\;
\begin{cases}
\;\Phi(\ell, s)\,P_-(s), & \ell \ge s,\\[2pt]
\,-\,\Phi(\ell, s)\,P_+(s), & \ell < s,
\end{cases}
\qquad
(\mathcal G h)(\ell) \;:=\; \sum_{s} G(\ell, s{+}1)\, h(s),
\]
where the backward branch runs along the persistent bundle (on which
the cocycle is invertible), and
$\|\mathcal G\|_{\infty\to\infty} \le C\,(m_-^{-1} + m_+^{-1})$: the
dichotomy margins are exactly what makes the Green operator bounded.
The classical fact~\citep{coppel1978dichotomies} is that
$u = \mathcal G h$ is the \emph{unique bounded} solution of the forced
linear system $u(\ell+1) = (I+A(\ell))u(\ell) + h(\ell)$ --- and this
is the non-autonomous replacement for ``solve $Au + h = 0$''.

\emph{Why a Green kernel.} A naive solution formula for the forced
system would be $\sum_{s \le \ell} \Phi(\ell, s)\,h(s)$. But this
formula \emph{diverges} under the dichotomy (A4): the persistent part
of $h$ lies in $V_\parallel(s)$ (rate $r_+ > 1$), and its forward
propagation $\Phi(\ell,s)P_+(s)h(s)$ grows geometrically. The
dichotomy-based Green kernel extracts the bounded parts: the forward
branch uses $P_-(s)$ to propagate only the decaying subspace (rate
$r_- < 1$), and the backward branch uses $P_+(s)$, with the negative
sign ensuring backward decay. The sum $\mathcal G h$ is the unique
bounded solution --- a standard construction in dichotomy
theory~\citep{coppel1978dichotomies}.

\begin{theorem}[Distinguished bounded trajectory]\label{thm:fixed:eq}
Assume (A4) and $\psi \in C^1$. Then there exist $\varrho_0 > 0$ and,
for $|\varrho| < \varrho_0$, a unique bounded solution
$u^\star_\varrho(\cdot)$ of \eqref{eq:closedloop} depending $C^1$ on
$\varrho$, with $u^\star_0 \equiv 0$ and\footnote{In the weight-tied specialisation the Green operator applied to the
constant forcing telescopes to $-A^{-1}$, and \eqref{eq:eq-branch}
collapses to the familiar equilibrium branch
$u^\star(\varrho) = -\varrho A^{-1}\psi(0) + O(\varrho^2)$; under
(A4$^\circ$) it is the classical LTI steady state
$h^\star = (I - \bar A)^{-1}\,\bar B\,e$ of the stable looped
architecture, reached from any initial state.}
\begin{gather}\label{eq:eq-branch}
  u^\star_\varrho(\ell) \;=\; \varrho\,(\mathcal G\,\psi(0))(\ell) \;+\; O(\varrho^2)
  \quad\text{uniformly in } \ell,\\
  \sup_\ell\|u^\star_\varrho(\ell)\| \;\le\; \frac{C\,|\varrho|\,\|\psi(0)\|}{\min(m_-, m_+)} + O(\varrho^2).\notag
\end{gather}
\end{theorem}
\begin{proof}
See \Cref{app:partIII} (\Cref{pf:fixed-eq}); a contraction argument in
the space of bounded sequences, with $\mathcal G$ providing the
inverse of the linear part.
\end{proof}

Two reference objects now coexist and must not be conflated:
$u \equiv 0$ is the \emph{unforced} homogeneous trajectory along which
the cocycle was linearised, and $u^\star_\varrho$ is the \emph{forced}
distinguished trajectory the gate creates.

\begin{proposition}[The geometric channel]\label{prop:shift}
Under the assumptions of \Cref{thm:fixed:eq}, the first-order
sensitivity of the distinguished trajectory to the gate amplitude is
\[
\frac{d}{d\varrho} u^\star_\varrho(\ell) \Big|_{\varrho=0} \;=\; (\mathcal G\, \psi(0))(\ell),
\]
and the persistent component of the displacement --- the
\emph{geometric channel}
$d_+(\varrho) := \sup_\ell \| P_+(\ell)\, u^\star_\varrho(\ell) \|$ ---
satisfies
\[
d_+(\varrho) \;=\; |\varrho|\,\sup_\ell \bigl\| P_+(\ell)\,(\mathcal G\,\psi(0))(\ell) \bigr\| \;+\; O(\varrho^2).
\]
Both are generically nonzero: $u^\star_\varrho$ moves for
$\varrho \ne 0$ iff the constant forcing $\psi(0)$ is not annihilated
by $\mathcal G$ (in the weight-tied specialisation, iff
$\psi(0) \ne 0$), and the gate drives $d_+$ nontrivially iff
$P_+\,\mathcal G\,\psi(0)$ is not identically zero.
\end{proposition}
\begin{proof}
Differentiate \eqref{eq:eq-branch} and take norms.
\end{proof}

The proposition's content is the direction and slope of the
displacement at the origin --- first-order identification theory, not
the global trajectory. This is the observable that
\Cref{thm:fastmode} protected: the displacement is insensitive to
fast-mode weight perturbations, which is what makes it usable for
identification rather than a curiosity of the linear model. Two scope
notes complete the reading. The formula \eqref{eq:eq-branch} is the
quasi-steady-state object of the closed loop --- the same
identification the slow-manifold reduction produces, packaged as a
bounded solution rather than an equilibrium. And it lives in the
linearised, unconstrained model: what the rest of the network reads is
the normalised observable
$\mathcal S(x^\star(\ell) + u^\star_\varrho(\ell))$, and once the
displacement is comparable to the scale of $\Mcalnorm$ the constrained
trajectory departs from the affine formula --- \eqref{eq:eq-branch} is
a direction and a slope, not a global location.

\subsection{Stability of the closed loop}\label{sec:fixed:stability-sec}

The gate does not merely displace the trajectory: because $\psi$ reads
the pooled state back, it closes a feedback loop, and a closed loop
can destroy the dichotomy. The closed-loop generator family is
$A(\ell) + \varrho\,D\psi(u^\star_\varrho(\ell))$; the perturbation
$\varrho\,D\psi$ is the feedback through the linearised injection ---
note that with the loop open (frozen forcing $\psi$ evaluated along a
fixed sequence) the dichotomy would be untouched at any $\varrho$; it
is precisely the feedback that puts a price on the gate. The right
tool is the \emph{roughness theorem} for exponential dichotomies: a
dichotomy survives any perturbation of the generator family that is
uniformly small compared to its margins.

\begin{theorem}[Stability margin of the closed loop]\label{thm:fixed:stability}
Assume (A4) and $\psi \in C^1$ with
$L_\psi := \sup_\ell \|D\psi(u^\star_\varrho(\ell))\| < \infty$. There
exists an absolute constant $c_0 \in (0, 1]$ such that, for any gate
amplitude satisfying\footnote{The bound $\varrho_{\max}$ is the
largest gate amplitude that preserves the saddle structure; it scales
linearly with the dichotomy margins (more hyperbolic $\Rightarrow$
larger margin $\Rightarrow$ bigger admissible gate) and inversely with
the dichotomy constant (more non-normal / more depth-incoherent
$\Rightarrow$ smaller admissible gate).}
\begin{equation}\label{eq:rho-max}
  |\varrho| \;<\; \varrho_{\max} \;:=\; c_0\,\frac{\min(m_-,\, m_+)}{C^2\, L_\psi},
\end{equation}
the closed-loop cocycle generated by
$A(\ell) + \varrho\,D\psi(u^\star_\varrho(\ell))$ admits an
exponential dichotomy with the same bundle dimensions as (A4),
projectors within $O(C^2|\varrho|L_\psi/\min(m_-,m_+))$ of
$P_\pm(\ell)$, and rates degraded by at most $O(C|\varrho|L_\psi)$:
the persistent/decaying geometry is qualitatively preserved, and the
identification analysis (\Cref{thm:fastmode}, \Cref{thm:geom},
\Cref{prop:shift}) is well-defined inside the feasibility region
\eqref{eq:rho-max}.
\end{theorem}
\begin{proof}
See \Cref{app:partIII} (\Cref{pf:fixed-stability}); the discrete-time
roughness theorem for exponential
dichotomies~\citep{coppel1978dichotomies,kloeden2011nonautonomous},
applied to the perturbation family
$\varrho\,D\psi(u^\star_\varrho(\ell))$.
\end{proof}

\Cref{thm:fixed:eq} computed the distinguished trajectory under the
assumption that the dichotomy of the open-loop cocycle persists; the
present theorem establishes \emph{when} that assumption is justified
for the closed loop. It is the bridge between the open-loop analysis
and the identification framework: \Cref{thm:fastmode},
\Cref{thm:geom}, and \Cref{prop:shift} all rely on the closed-loop
dichotomy being well-defined, and \eqref{eq:rho-max} is the explicit
feasibility region.

\begin{remark}[Weight-tied specialisation]\label{rem:weight-tied-stability}
In the weight-tied autonomous specialisation $A(\ell) \equiv A$ with
$A$ diagonalisable and real spectrum in $(-2,0)\cup(0,\infty)$, the
stable margin\footnote{The stable margin is 
$m_- = \min_{\sigma\in\Sigma^-}\;\min\bigl(|\sigma|,\; 2 -
|\sigma|\bigr)$, where $\Sigma^-$ represents the eigenvalues
with $\sigma < 0$.} is a decaying mode loses discrete stability either by
drifting to the marginal point $+1$ (as $\sigma \to 0^-$) or by
crossing the flip point $-1$ (as $\sigma \to -2^+$), and the
admissible gate for the most negative eigenvalue is
$|\varrho| \lesssim (2 - |\sigma|)/L_\psi$. The flip boundary is an
artefact of the $\Delta\ell = 1$ discretisation --- a continuous-time
reading would give $m_- = \min_{\sigma\in\Sigma^-}|\sigma|$ --- and a
network is a discrete stack, so the discrete criterion is the
operative one; but the two readings should not be conflated when
reasoning in the ODE picture.
\end{remark}

\emph{Stability versus usefulness.} \eqref{eq:rho-max} \emph{bounds}
the admissible gate amplitude but does not \emph{select} which value
training settles on. Inside the bound, the geometric channel
$d_+(\varrho)$ (\Cref{prop:shift}) is well-defined and the
identification analysis applies; the selection of the trained
$\varrho$ --- where inside the bound the gate ends up --- is governed
by whether opening it reduces the loss on the training distribution,
which is a property of the data. Stability delimits the feasible
region; usefulness picks the point in it, and usefulness is the
subject of \S\ref{sec:fixed:load}.

\subsection{When is the slow path load-bearing: a block-level latent model}\label{sec:fixed:load}

The architecture-side analysis is complete: past saturation the fast
path factors through the block summary (\Cref{thm:duality}), the gate
visibly displaces the distinguished trajectory (\Cref{prop:shift}),
and the loop keeps its dichotomy for $|\varrho| < \varrho_{\max}$.
What the analysis cannot decide is whether opening the gate
\emph{helps}, because that depends on the data. This subsection states
the data-side condition in the most elementary form we know --- one
latent variable per block, two propositions, no further machinery.

Let blocks $B_k := (x_{(k-1)P+1}, \ldots, x_{kP})$ be generated by a
latent chain: a latent state $z_k$ evolves as a Markov chain with
kernel $Q(z_{k+1}\mid z_k)$, and given the latents the blocks are
drawn independently, $B_k \sim e(\cdot\mid z_k)$. The latent is the
formalisation of ``cross-block structure'' --- whatever persists from
one block to the next: topic, style, regime, the state of a described
world --- and the degenerate cases are $z_k \equiv z$ (constant:
nothing to track) and $z_k$ i.i.d.\ (nothing persists: the past
reveals nothing about the future latent).

\begin{proposition}[No cross-block structure $\Rightarrow$ the gate is not identifiable from the risk]\label{prop:noslow}
Suppose the latent chain is degenerate in either sense, so that the
blocks are mutually independent. Consider next-token prediction within
block $k+1$, and any predictor family in which the cross-block channel
enters only through a measurable functional of the previous blocks
(as in (C2)). Then the Bayes-optimal predictor is block-local, the
minimal risk is attained at $\varrho = 0$, and the risk is flat in
$\varrho$ at that optimum: $\varrho$ is not identifiable from the
prediction risk.
\end{proposition}
\begin{proof}
See \Cref{app:partIII} (\Cref{pf:noslow}). Under independence the
conditional law of $B_{k+1}$ given the past equals its marginal, so
conditioning on any functional of past blocks leaves the predictive
law unchanged.
\end{proof}

The gate, in other words, can sit at zero for two entirely different
reasons --- because stability forbids opening it, or because nothing
rewards opening it --- and the framework insists on keeping them
apart: the first is \eqref{eq:rho-max}, the second is block
independence.

The next proposition is the statistical backing for the block-mean
choice in (C2): under exponential-family emissions the block mean is
not one heuristic among many but the sufficient compression.

\begin{proposition}[Exponential-family emissions make the block mean sufficient]\label{prop:suffstat}
Suppose the emission is a product of per-token exponential-family
densities,
$e(B_k\mid z_k) = \prod_t h(x_t)\exp(\langle T(x_t), \theta(z_k)\rangle
- A(\theta(z_k)))$, for some statistic $T$. Then the block average
$\bar T_k := P^{-1}\sum_t T(x_t)$ is sufficient for $z_k$ given $B_k$;
if $T$ is linear, block-mean pooling computes it
exactly~\citep{lehmann1998theory}.
\end{proposition}
\begin{proof}
See \Cref{app:partIII} (\Cref{pf:suffstat}); the factorisation
criterion.
\end{proof}

So within this model class the slow path's compression loses nothing
that \emph{any} cross-block channel could have used --- the block mean
is the first moment of the block's empirical measure and the
sufficient statistic for exponential-family structure. The two
propositions bracket the design from below and above: when blocks are
independent, no cross-block channel of any form earns its parameters;
when blocks share a persistent latent with exponential-family
emissions, the canonical slow path --- block-mean compression, gated
injection --- carries exactly the information a Bayes predictor of the
next block needs. The slow coupling is load-bearing precisely when the
data has a nondegenerate latent, and the persistence of that latent
(the spectral gap of $Q$) plays, on the data side, the same role the
timescale ratio $\varepsilon$ plays on the architecture side: the
two-clock design is an inductive bias matched to two-timescale data.

\section{Discussion}\label{sec:discussion}

\textbf{Scaling laws, conditionally.} Empirical scaling
laws~\citep{kaplan2020scaling,hoffmann2022chinchilla} are commonly
read as architecture-free regularities. The present framework suggests
a conditional reading, offered as a heuristic implication rather than
a proved result. By \S\ref{sec:fixed:load}, whether the slow path's
parameters contribute to the achievable risk depends on whether the
data carries cross-block latent structure: on data without it they are
dead weight (\Cref{prop:noslow}) and cannot buy a better exponent; on
data with it they open the only channel through which the relevant
structure can be used at all. A scaling exponent measured on one
corpus and trusted as universal is then an artefact of that corpus's
latent structure --- an illusion of universality. The framework makes
the claim testable in principle: vary the latent persistence in
synthetic data, watch whether the trained gate opens and the exponent
responds. We regard carrying out that test as the natural empirical
sequel to this paper.

\textbf{Which hypotheses can be built in.} The assumptions (A1)--(A7)
are not uniformly empirical: the analysis of stable looped
architectures (\S\ref{sec:intro:ode}) shows that a parameterisation
can \emph{enforce} part of the list. In the LTI-parameterised looped
regime~\citep{geiping2025scaling,prairie2026parcae}, the stable half
of the dichotomy is structural --- $r_- < 1$ holds for any parameter
values, $C = 1$ under weight tying, and the alignment condition (A7)
is exact because the projectors are coordinate projections --- while
the same parameterisation \emph{forbids} the saddle half: the
persistent bundle is trivial and all persistence is carried by the
forcing, i.e.\ by the second coupling. The two extremes bracket the
design space of this paper --- the standard Transformer
($\varrho = 0$) must keep its persistent content inside the cocycle's
$V_\parallel$; the looped LTI architecture keeps none of it there and
routes everything through the injection --- and the empirical
stability record of the latter reads naturally in our vocabulary:
training diverges exactly when the learned spectral radius crosses the
dichotomy boundary, so the margin $m_-$ is the operative stability
frontier of real training, not a modelling convenience. What no
parameterisation so far enforces is a nontrivial persistent bundle
with a controlled margin $m_+$ --- the saddle half of (A4) --- and
that is the precise point at which the trained network enters the
analysis as an empirical object.

\textbf{Identification in practice.} The identified directions of
\S\ref{sec:coupling:ident} correspond to estimable quantities: the
transverse decay profile is the depthwise rate at which token states
homogenise (a rank-collapse rate, measurable from activations); the
trajectory displacement is measurable by intervening on the gate
($\varrho \to 0$ at inference) and recording the shift of the
late-depth state along the persistent bundle; the kernel universality
of \Cref{prop:kernel} predicts that late-depth attention maps lose
their query--key specificity, layer after layer; and in looped
architectures the approach to the distinguished trajectory is directly
visible as the saturating improvement of quality with test-time loop
count --- the shape of $\|h_t - h^\star\| \lesssim r_-^{\,t}$. Because
the gate enters the risk flatly under block independence
(\Cref{prop:noslow}) and moves the trajectory visibly otherwise
(\Cref{prop:shift}), the trained gate amplitude itself functions as a
built-in diagnostic for cross-block structure in the training
distribution. None of this requires access to the training process ---
these are observables of the trained artefact, which is what makes the
identification framing operational. The one genuinely new estimation
problem the non-autonomous framing raises is the dichotomy constant
$C$: it enters every bound, it absorbs both non-normality and the
depth-incoherence of the splitting, and unlike a spectrum it is not
readable from any single layer. Estimating dichotomy spectra and
constants from activation trajectories is, to our knowledge, open, and
it is the natural statistical companion to this paper.

\textbf{Limitations.} Four are structural. The analysis is local --- a
linearisation along a reference trajectory, valid in the neighbourhood
of (A2) and only past the transient regime
$\ell \gtrsim \log C/\!\log(1/r_-)$; the persistent bundle exits any
fixed neighbourhood at its dichotomy rate, so the theory describes the
routing of information near saturation, not global trajectories. The
hypotheses that tie the analysis to the trained network split unevenly
between enforceable and empirical: the stable margin, the constant,
and the alignment can be made structural by parameterisation (previous
paragraph), while the saddle half of (A4), the transverse gap of (A5)
for the full (non-skeleton) layer map, and the training-alignment
hypothesis (A6) remain statements about what training produces ---
weaker than their autonomous ancestors, hence more credible, but
harder to estimate, per the point about $C$ above. The block-local
model class (C1) idealises full-context attention; the leaky variant
adds a cross-block term of the order of the attention mass placed
outside the block, and quantifying its interaction with the slow path
is open. And the proximity of trained networks to the homogeneous
reference --- the empirical content of saturation --- rests on the
rank-collapse and clustering literature
\citep{dong2021rank,geshkovski2023clusters}, whose sharpest results
concern the attractive, MLP-free, non-causal idealisation; the causal,
MLP-including case is exactly where our hypotheses live and their
theorems do not yet reach --- in both directions: their convergence
machinery cannot yet discharge our (A2)-proximity, and their model,
having no persistent bundle, cannot even express our (A4).

%% file: Sec-appendix.tex
\section{Proofs for Part I}\label{app:partI}

Throughout the appendices, $\|\cdot\|$ denotes the Euclidean norm on
$\R^d$, the induced operator norm on matrices, and the supremum norm
$\|u\|_\infty = \sup_\ell \|u(\ell)\|$ on bounded sequences; which one
is meant is always determined by the argument. Constants named $C'$,
$C''$, $c_0$ are absolute or depend only on the quantities displayed
with them. We freely use the standing assumptions (A1)--(A3) and, where
stated, (A4)--(A7) and the dichotomy notation of the main text
($\Phi, \Phi_j, P_\pm(\ell), C, r_\mp, m_\mp, C_f, r_f,
\kappa_\alpha$).

\subsection{Smoothness of the two normalisation placements}\label{app:lnlemmas}

The two lemmas cited in \S\ref{sec:intro:preln} formalise the
regularity difference between Pre-LN and Post-LN. Write $\mathcal S$
for a norm-scaling operator that is $C^k$ ($k \ge 1$) on an open set
$U \subseteq \R^d$ and only continuous (or undefined) on the closed
complement $N := \R^d \setminus U$. For the stabilised RMSNorm
$\mathcal S(x) = g \odot x / \sqrt{\|x\|^2/d + \epsilon_S}$ with
$\epsilon_S > 0$ one has $U = \R^d$ and $N = \emptyset$; for the
unstabilised version ($\epsilon_S = 0$), $N = \{0\}$; for selector
instances, $N$ is the union of the selection boundaries. One layer
consists of the two pre- or post-normalised sublayers of
\eqref{eq:layer_evolution}; we state the lemmas for a single sublayer
map, the composition being handled by the chain rule.

\begin{lemma}[Pre-LN regularity]\label{lem:preln-smooth}
Let $\phi \in C^k$ and let $F(x) := x +
\mathrm{Sublayer}(\mathcal S(x))$ be a pre-normalised sublayer. Then
$F \in C^k$ on the open set $\{x : x \in U\}$; in particular, with
stabilised RMSNorm, $F \in C^k(\R^{T\times d})$. The singular set of
$F$ is $N$ itself: a fixed, weight-independent subset of the input
space.
\end{lemma}
\begin{proof}
$F$ is the sum of the identity and the composition
$\mathrm{Sublayer} \circ \mathcal S$. The sublayer (attention with
softmax scores, or the MLP) is $C^k$ jointly in its inputs, softmax
being $C^\infty$ and $\phi \in C^k$ by hypothesis; $\mathcal S$ is
$C^k$ on $U$. The chain rule gives $F \in C^k$ wherever the inner
argument lies in $U$, i.e.\ on $U$ itself, and the location of the
kink set does not involve the weights because $\mathcal S$ is applied
\emph{before} any learned map.
\end{proof}

\begin{lemma}[Post-LN regularity]\label{lem:postln-cont}
Let $\phi \in C^k$ and let
$G(x) := \mathcal S\bigl(x + \mathrm{Sublayer}(x)\bigr)$ be a
post-normalised sublayer. Then $G$ is $C^k$ on the open set
$\{x : x + \mathrm{Sublayer}(x) \in U\}$ and continuous wherever
$\mathcal S$ is; globally it is only piecewise $C^k$ when
$N \ne \emptyset$, and its singular set
$\{x : x + \mathrm{Sublayer}(x) \in N\}$ is the preimage of $N$ under
a \emph{weight-dependent} map --- it moves as the network trains.
\end{lemma}
\begin{proof}
Same chain-rule argument with the composition order reversed:
$\mathcal S$ is now applied \emph{after} the learned update, so the
requirement ``inner argument in $U$'' reads
$x + \mathrm{Sublayer}(x) \in U$, a condition on a weight-dependent
image. Continuity off that set is the continuity of the composition;
piecewise regularity is the decomposition of the input space by the
closed preimage of $N$.
\end{proof}

The pair explains the asymmetry asserted in the main text: both
placements are smooth for stabilised norms, but when $\mathcal S$
carries a singular set, Pre-LN's kinks sit at fixed locations of the
input space while Post-LN's re-projection drags them through the
learned update --- which is why the forward-Euler reading, built on
linearising at a reference trajectory, prefers Pre-LN.

\subsection{Proof of Proposition~\ref{prop:bounded}}\label{pf:bounded}

Write $z = \mathcal S(x) \in \Mcalnorm$, so $\|z\| \le \sqrt d$ (unit
gains; general gains multiply the bound by $\|g\|_\infty$ and are
absorbed into the constants). For the attention half-step, the update
of token $t$ is $\sum_{s\le t} \alpha_{ts} W_V^{(\ell)} z_s$, a convex
combination of the vectors $W_V^{(\ell)} z_s$, hence of norm at most
$\|W_V^{(\ell)}\| \sqrt d$. For the MLP half-step, the input is again
a normalised state $z' \in \Mcalnorm$, and
\[
\bigl\|W_2^{(\ell)} \phi(W_1^{(\ell)} z' + b_1^{(\ell)}) + b_2^{(\ell)}\bigr\|
\;\le\; \|W_2^{(\ell)}\| \bigl( \|\phi(W_1^{(\ell)} z' + b_1^{(\ell)}) - \phi(0)\| + \|\phi(0)\| \bigr) + \|b_2^{(\ell)}\|,
\]
and $\|\phi(W_1 z' + b_1) - \phi(0)\| \le L_\phi (\|W_1^{(\ell)}\|\sqrt
d + \|b_1^{(\ell)}\|)$ by the Lipschitz property of $\phi$ on the
bounded set swept by $W_1 z' + b_1$. Adding the two half-step bounds
gives $\sup_x \|\mathcal N^{(\ell)}_t(x)\| \le C_{\mathcal
N}^{(\ell)}$ as displayed. The stream bound follows by telescoping
\eqref{eq:one-step}:
$\|x^{(\ell)}\| \le \|x^{(0)}\| + \sum_{r<\ell} \sup_x \|\mathcal
N^{(r)}(x)\| \le \|x^{(0)}\| + \sum_{r<\ell} C_{\mathcal N}^{(r)}$.
\qed

\subsection{Proof of Theorem~\ref{thm:euler}}\label{pf:euler}

Let $\Phi_c(\ell + \Delta\ell, \ell)$ denote the exact propagator of
$\dot{\delta x} = A(s)\,\delta x$ on $[\ell, \ell+\Delta\ell]$ --- the
time-ordered exponential, given by the Peano--Baker series
\[
\Phi_c(\ell+\Delta\ell, \ell)
\;=\; I + \int_\ell^{\ell+\Delta\ell} A(s_1)\, ds_1
\;+\; \sum_{n\ge2} \int_\ell^{\ell+\Delta\ell}\!\!\int_\ell^{s_n}\!\!\cdots\int_\ell^{s_2} A(s_n)\cdots A(s_1)\, ds_1\cdots ds_n,
\]
which converges absolutely under (A1) with
$\|n\text{-th term}\| \le (a\Delta\ell)^n / n!$. Comparing with the
Euler factor $I + \Delta\ell\,A(\ell)$:
\[
\bigl\|\Phi_c - (I + \Delta\ell A(\ell))\bigr\|
\;\le\; \Bigl\|\int_\ell^{\ell+\Delta\ell} \bigl(A(s) - A(\ell)\bigr)\,ds\Bigr\|
\;+\; \sum_{n\ge2} \frac{(a\Delta\ell)^n}{n!}
\;\le\; \omega_A(\Delta\ell)\,\Delta\ell \;+\; \tfrac12 a^2 \Delta\ell^2\, e^{a\Delta\ell},
\]
using the modulus of continuity for the first term and
$e^x - 1 - x \le \tfrac12 x^2 e^x$ for the tail. Applying both
operators to $\delta x(\ell)$ gives the displayed estimate with the
implied constant $e^{a\Delta\ell}/2 \le e^a/2$ for $\Delta\ell \le 1$.
At $\Delta\ell = 1$ and $a \sim 1$ the bound is $O(1)$, and it is
$O(a^2 + \omega_A(1))$ in the saturation regime, which is the
quantitative claim. \qed

\subsection{Proof of Theorem~\ref{thm:eigenvalue}}\label{pf:eigenvalue}

The interval structure of the dichotomy spectrum is the discrete-time
Sacker--Sell spectral
theorem~\citep{sackersell1978spectral,kloeden2011nonautonomous}; we
organise the proof into six steps, proving in full the steps specific
to our setting (4--6) and citing the classical structure theorems for
steps 1--3.

\emph{Step 1 (the resolvent set is open).} If the rescaled cocycle
$\Phi_\gamma(\ell,s) := \gamma^{-(\ell-s)}\Phi(\ell,s)$ admits an
exponential dichotomy, then so does $\Phi_{\gamma'}$ for all $\gamma'$
in a neighbourhood of $\gamma$: rescaling by $(\gamma/\gamma')^{\ell-s}$
perturbs the rates multiplicatively, and rates strictly inside the
unit circle (resp.\ outside) stay there for $\gamma'/\gamma$ close to
$1$. Hence $\Sigma_{\mathrm{dich}}$ is closed.

\emph{Step 2 (the spectrum is bounded).} For
$\gamma > 1 + a \ge \sup_\ell \|I + A(\ell)\|$, the rescaled cocycle
is uniformly contracting ($\|\Phi_\gamma(\ell,s)\| \le ((1+a)/\gamma)^{
\ell-s}$), which is a dichotomy with $P_- = I$; so
$\Sigma_{\mathrm{dich}} \subseteq (0,\, 1+a]$. (If the factors
$I + A(\ell)$ are not all invertible the spectrum may extend down to
$0$; this affects nothing below, where only the position of
$\Sigma_{\mathrm{dich}}$ relative to $1$ is used.)

\emph{Step 3 (at most $d$ intervals, with a bundle filtration).} On
each connected component of the resolvent set the rank of the
dichotomy projector $P_-^{(\gamma)}$ is constant, and it is monotone
non-decreasing in $\gamma$ (a dichotomy at scale $\gamma$ with stable
rank $\rho$ forces stable rank $\ge \rho$ at any larger resolvent
scale, since the same $\rho$-dimensional bundle still contracts after
rescaling). A rank function taking values in $\{0, \ldots, d\}$ can
jump at most $d$ times, so the complement --- the spectrum --- has at
most $d$ connected components, each a compact interval
$[\gamma_i^-, \gamma_i^+]$; the associated spectral bundles are the
intersections of the stable bundle just above the interval with the
unstable bundle just below it. For the full argument see
\citep[Ch.~5]{kloeden2011nonautonomous} and
\citep{sackersell1978spectral}.

\emph{Step 4 (growth rates lie in the intervals).} Let $v \ne 0$
belong to the spectral bundle of $[\gamma_i^-, \gamma_i^+]$ at time
$0$. For any $\gamma > \gamma_i^+$ in the resolvent set, $v$ lies in
the $\gamma$-stable bundle, so
$\|\Phi(\ell,0)v\| \le C_\gamma\, (\gamma r_{-,\gamma})^{\ell}\|v\|$
with $\gamma r_{-,\gamma} < \gamma$; letting $\gamma \downarrow
\gamma_i^+$ gives $\limsup_\ell \|\Phi(\ell,0)v\|^{1/\ell} \le
\gamma_i^+$. Symmetrically, for $\gamma < \gamma_i^-$, $v$ lies in the
$\gamma$-unstable bundle and the backward dichotomy bound forces
$\liminf_\ell \|\Phi(\ell,0)v\|^{1/\ell} \ge \gamma_i^-$. Solutions in
bundles whose interval lies below $1$ therefore decay, and above $1$
grow, exponentially.

\emph{Step 5 (hyperbolicity).} By definition,
$1 \notin \Sigma_{\mathrm{dich}}$ means the unrescaled cocycle admits
an exponential dichotomy, i.e.\ exactly (A4) with
$r_- < 1 < r_+$ read off from the resolvent intervals adjacent to
$1$; conversely a dichotomy at scale $1$ puts $1$ in the resolvent
set.

\emph{Step 6 (weight-tied collapse).} Let $A(\ell) \equiv A$ with $A$
diagonalisable, and let $\{|1+\sigma_i|\}$ be the moduli of the
eigenvalues of $I + A$. If $\gamma \ne |1+\sigma_i|$ for all $i$, the
spectral projections of $I+A$ onto the eigenvalues of modulus
$< \gamma$ and $> \gamma$ furnish a dichotomy for
$\gamma^{-(\ell-s)}(I+A)^{\ell-s}$ (with constant equal to the
eigenbasis condition number), so $\gamma$ is in the resolvent set.
Conversely, if $\gamma = |1+\sigma_i|$ for some $i$, pick an
eigenvector $v$; then $\|\gamma^{-\ell}(I+A)^\ell v\| = \|v\|$ for
all $\ell$, a bounded, non-decaying, non-growing solution, which is
incompatible with either dichotomy bound; so
$\gamma \in \Sigma_{\mathrm{dich}}$. Hence
$\Sigma_{\mathrm{dich}} = \{|1+\sigma_i|\}$ and the bundles are the
corresponding eigenspaces. \qed

\section{Proofs for Part II}\label{app:partII}

\subsection{Toolbox: the Green operator of a dichotomy}\label{app:green}

The following classical lemma~\citep[Ch.~3]{coppel1978dichotomies} is
used by \Cref{pf:duality,pf:fastmode} below and by all of
\Cref{app:partIII}. We state it for two-sided depth $\ell \in \Z$; the
network's finite depth window $\{0, \ldots, L\}$ is embedded by
extending $A(\ell)$ constantly outside the window (which preserves
(A1) and (A4)), and the finite-depth solution differs from the
two-sided formula by boundary terms of size
$O(C r_-^{\ell} + C r_+^{\ell - L})$, negligible in the regime
$\log C / \log(1/r_-) \lesssim \ell \lesssim L$ that all main-text
statements concern.

\begin{lemma}[Green operator]\label{lem:green}
Assume (A4) and define
\[
G(\ell, s) \;=\;
\begin{cases}
\;\Phi(\ell, s)\,P_-(s), & \ell \ge s,\\
\,-\,\Phi(\ell, s)\,P_+(s), & \ell < s,
\end{cases}
\qquad
(\mathcal G h)(\ell) \;:=\; \sum_{s \in \Z} G(\ell, s{+}1)\, h(s),
\]
where for $\ell < s$ the cocycle acts along the persistent bundle, on
which it is invertible, via
$\Phi(\ell,s)P_+(s) := \bigl(\Phi(s,\ell)|_{V_\parallel(\ell)}\bigr)^{-1}P_+(s)$.
Then:
\begin{enumerate}
\item $\|\mathcal G\|_{\infty\to\infty} \le C\,(m_-^{-1} + m_+^{-1})$;
\item for every bounded sequence $h$, $u := \mathcal G h$ is a
  bounded solution of $u(\ell+1) = (I + A(\ell))\,u(\ell) + h(\ell)$;
\item it is the \emph{only} bounded solution.
\end{enumerate}
\end{lemma}
\begin{proof}
(1) Split the sum at $s = \ell - 1$: the forward branch contributes
$\sum_{s \le \ell-1} \|\Phi(\ell, s{+}1)P_-(s{+}1)\| \le
\sum_{k\ge0} C r_-^{k} = C/m_-$, the backward branch
$\sum_{s \ge \ell} \|\Phi(\ell, s{+}1)P_+(s{+}1)\| \le
\sum_{k\ge1} C r_+^{-k} = C/(r_+ - 1) = C/m_+$.

(2) Write $u(\ell) = \sum_{s\le\ell-1}\Phi(\ell,s{+}1)P_-(s{+}1)h(s)
- \sum_{s\ge\ell}\Phi(\ell,s{+}1)P_+(s{+}1)h(s)$. Applying
$(I+A(\ell))$ and using the cocycle property
$\Phi(\ell{+}1, s{+}1) = (I+A(\ell))\,\Phi(\ell, s{+}1)$ turns each
term of $u(\ell)$ into the corresponding term of $u(\ell{+}1)$ except
at the splitting index, where the forward branch of $u(\ell{+}1)$
gains the term $\Phi(\ell{+}1,\ell{+}1)P_-(\ell{+}1)h(\ell) =
P_-(\ell{+}1)h(\ell)$ and the backward branch loses
$-\Phi(\ell{+}1,\ell{+}1)P_+(\ell{+}1)h(\ell) =
-P_+(\ell{+}1)h(\ell)$; the two boundary terms add to
$(P_- + P_+)(\ell{+}1)\,h(\ell) = h(\ell)$.

(3) The difference $v$ of two bounded solutions solves the homogeneous
system, $v(\ell) = \Phi(\ell,s)v(s)$ for $\ell \ge s$. For the
persistent part, $P_+(\ell)v(\ell) = \Phi(\ell,s)P_+(s)v(s)$ with
$s > \ell$ gives $\|P_+(\ell)v(\ell)\| \le C r_+^{\ell-s}\|v\|_\infty
\to 0$ as $s \to +\infty$, so $P_+ v \equiv 0$; for the decaying part,
$\|P_-(\ell)v(\ell)\| = \|\Phi(\ell,s)P_-(s)v(s)\| \le
C r_-^{\ell-s}\|v\|_\infty \to 0$ as $s \to -\infty$. Hence
$v \equiv 0$.
\end{proof}

\subsection{Proof of Theorem~\ref{thm:duality}}\label{pf:duality}

Throughout, fix the held state $\hat y$, so the forcing
$\varrho\,\psi(\hat y_{k-1})$ is a constant (block-indexed) sequence
and the linearised full flow $\Phi^{\mathrm{full}}_\ell$ is affine;
differences of solutions evolve under the homogeneous linear flow.
Within each block the frozen kernel is the causal-uniform matrix on
$P$ positions, whose eigenvalues $1, 1/2, \ldots, 1/P$ are
\emph{distinct}; it is therefore diagonalisable with an eigenbasis
$\{w_j\}_{j=1}^P$ of finite condition number $\kappa_\alpha$, which
justifies the sector decomposition used below and fixes the constant
$\kappa_\alpha$ of the statement.

\emph{Part 1 (exactness on the manifold).} Under (C1), attention is
confined to the current block, so the block map acts on each block as
an independent causal window; \Cref{lem:homog} applied blockwise shows
that if all tokens of block $k$ share the value $y_k$, they share a
common value after the layer as well. The injection
$\varrho\,\psi(\hat y_{k-1})$ is constant within block $k$ by (C2), so
it preserves within-block homogeneity too. Hence
$\mathcal H_P$ is invariant, and on it the common value of block $k$
updates by the exact, block-decoupled nonlinear evolution
$y_k \mapsto y_k + \bar{\mathcal N}^{(\ell)}(y_k) +
\varrho\,\psi(\hat y_{k-1})$: shared weights make
$\bar{\mathcal N}^{(\ell)}$ the same \emph{map} for every block.
Passing to the displayed linear form requires a point of evaluation,
and one step must be made explicit: the block generator
$D\bar{\mathcal N}^{(\ell)}(y)$ depends on the state through
$D\mathcal S(y)$ and $D\phi(\cdot)$, so blocks sitting at
\emph{different} values $y_k$ would carry block-dependent generators
$A_k(\ell)$. We linearise every block at the \emph{common} reference
$x^\star(\ell)$ --- the reference configuration $X^\star(\ell)$ lies
in $\mathcal H \subset \mathcal H_P$, so its value is shared by all
blocks by construction --- which yields the displayed flow
$\dot y_k = A(\ell)\,y_k + \varrho\,\psi(\hat y_{k-1})$ with the
single generator
$A(\ell) = D\bar{\mathcal N}^{(\ell)}(x^\star(\ell))$ of
\eqref{eq:MV-def}. For a block at distance
$\|y_k - x^\star(\ell)\|$ from the reference, its true generator
differs by $A_k(\ell) - A(\ell) = O(\|y_k - x^\star(\ell)\|)$ (the
$C^1$ moduli of (A3)), a discrepancy that is second order in the
deviation and is absorbed into the linearisation remainder of the
(A2) regime. The single-generator form is therefore the
linearised-at-reference reading --- exact on the reference trajectory,
first-order accurate in its neighbourhood --- and this is the sense in
which Part 1, and every downstream use of the tangent cocycle on
block coordinates, is asserted. Block-mean $\Pi$ restricted to
$\mathcal H_P$ reads off the common values, a bijection onto
$(y_1, \ldots, y_K)$ with inverse $\iota$.

\emph{Part 2 (decay off the manifold).} Decompose, block by block, the
deviation $\delta X$ in the eigenbasis of the block-local frozen
kernel: $\delta X = \delta X^{\mathcal H} + \delta X^\perp$, where
$\delta X^{\mathcal H}$ collects the sector-$1$ components (the
blockwise-homogeneous part) and $\delta X^\perp$ the sectors
$j \ge 2$. Since the flow is affine with forcing constant within
blocks --- hence supported in sector $1$ --- the difference
$\Phi^{\mathrm{full}}_\ell(\delta X) -
\Phi^{\mathrm{full}}_\ell(\delta X^{\mathcal H})$ equals the
homogeneous evolution of $\delta X^\perp$, which by
\Cref{thm:factor}(2) is sectorwise:
$\bigoplus_{j\ge2} \Phi_j(\ell, 0)$ applied to the fluctuation
coefficients. Passing to the eigenbasis and back costs at most
$\kappa_\alpha$, and each sector factor is bounded by
$C_f r_f^{\,\ell}$ by (A5); collecting constants (the equivalence
between the ambient norm and the maximum of blockwise sector norms is
absolute) gives
\[
\bigl\|\Phi^{\mathrm{full}}_\ell(\delta X) - \Phi^{\mathrm{full}}_\ell(\delta X^{\mathcal H})\bigr\|
\;\le\; C' \kappa_\alpha C_f\, r_f^{\,\ell}\, \|\delta X^\perp\|.
\]

\emph{Part 3 (duality).} Let $\tilde\Pi$ denote the sector-$1$
\emph{spectral} compression: blockwise, the coefficient of $w_1 =
\mathbf 1$ in the eigenbasis expansion, i.e.\ the biorthogonal
functional $\tilde\pi_1^\top$ applied to each block
($\tilde\pi_1^\top w_j = \delta_{1j}$). Three properties are
immediate: $\tilde\Pi = \Pi$ on $\mathcal H_P$ (both read the common
block value); $\tilde\Pi\,\delta X^\perp = 0$ by biorthogonality; and
$\|\tilde\Pi\| \le \kappa_\alpha$. Then
\[
\Phi^{\mathrm{full}}_\ell(\delta X)
\;=\; \Phi^{\mathrm{full}}_\ell(\delta X^{\mathcal H}) \;+\; E_\ell
\;=\; \iota\,\Phi(\ell,0)\,\tilde\Pi\,\delta X^{\mathcal H} \;+\; E_\ell
\;=\; \iota\,\Phi(\ell,0)\,\tilde\Pi\,\delta X \;+\; E_\ell,
\]
where the middle equality is Part 1 linearised, and
$\|E_\ell\| \le C'\kappa_\alpha C_f r_f^{\,\ell}\|\delta X^\perp\|$ by
Part 2. For $\ell \ge \tau_{\mathrm{sat}}(\epsilon_{\mathrm{tol}})$
the error is at most $\epsilon_{\mathrm{tol}}\|\delta X^\perp\|$,
which is the displayed formula (with the compression read as
$\tilde\Pi$; on $\mathcal H_P$, and for doubly stochastic kernels
everywhere, $\tilde\Pi = \Pi$). The bundle statement for one slow tick
is the dichotomy invariance
$P_\pm(\ell+L_r)\,\Phi(\ell+L_r,\ell) = \Phi(\ell+L_r,\ell)\,P_\pm(\ell)$
of (A4). \qed

\begin{remark}[Block-mean versus spectral compression, and where the
(A5) gap is used]\label{rem:duality-blockmean}
For the causal kernel the architectural block-mean $\Pi$ and the
spectral compression $\tilde\Pi$ differ off $\mathcal H_P$:
$\Pi\,\delta X^\perp \ne 0$ in general, and the one-shot formula
$\iota\,\Phi(\ell,0)\,\Pi$ would amplify that contamination along the
persistent bundle at rate $r_+^{\,\ell}$. The honest block-mean
statement factors through an intermediate depth: for
$\ell = \ell_0 + \ell_1$,
\[
\Phi^{\mathrm{full}}_\ell
\;=\; \iota\,\Phi(\ell, \ell_0)\,\Pi\,\Phi^{\mathrm{full}}_{\ell_0}
\;+\; E,
\qquad
\|E\| \;\le\; C''\,\kappa_\alpha\, C\, C_f\; r_+^{\,\ell_1}\, r_f^{\,\ell_0}\,\|\delta X\|,
\]
because after $\ell_0$ layers the state's fluctuation content has
decayed to $O(C_f r_f^{\ell_0})$, at which point block-mean and
spectral compression agree to the same order, and the residual error
is then amplified by at most $C r_+^{\ell_1}$. Choosing
$\ell_0 = \lceil \ell/2 \rceil$ and invoking the gap condition of
(A5), $r_f\,r_+ \le \theta < 1$, gives
$\|E\| \le C''\kappa_\alpha C C_f\,\theta^{\lfloor\ell/2\rfloor}
\|\delta X\|$ (for odd $\ell$ the leftover factor $r_f \le 1$ only
strengthens the bound): coarse-graining with the architectural
block-mean is
valid once it is inserted \emph{after} the transverse transient ---
measure after mixing. This is the one place in the paper where the
gap condition, rather than mere transverse contraction, is genuinely
load-bearing.
\end{remark}

\subsection{Proof of Theorem~\ref{thm:fastmode}}\label{pf:fastmode}

\emph{Step 1 (the frozen-forcing case; (A7) not needed here).} At
fixed $\hat y$ the forcing $h := \varrho\,\psi(\hat y)$ is a constant
sequence, and by \Cref{lem:green} the distinguished solution is
$u^\star_\varrho(\cdot\,; W) = \mathcal G_A\, h$, where $\mathcal G_A$
is the Green operator of the cocycle generated by $\{A(\ell)\}$.
Perturb $W$ by a fast-mode $\delta W$, inducing $\{\delta A(\ell)\}$
with $\delta A(\ell) = P_-(\ell{+}1)\,\delta A(\ell)$ and
$\eta = \sup_\ell\|\delta A(\ell)\|$. For
$\eta < \|\mathcal G_A\|^{-1}$, the perturbed bounded solution exists
and satisfies the fixed-point identity obtained by treating
$\delta A(\cdot)u(\cdot)$ as additional forcing:
\[
u_{\delta} \;=\; \mathcal G_A\bigl[h + \delta A(\cdot)\,u_\delta(\cdot)\bigr]
\;=\; u^\star + \sum_{n \ge 1} \bigl(\mathcal G_A\,\delta A\bigr)^n u^\star,
\]
the Neumann series converging geometrically with ratio
$\|\mathcal G_A\|\eta < 1$. In particular
\begin{equation}\label{eq:first-order-correction}
u_\delta - u^\star \;=\; \mathcal G_A\bigl[\delta A(\cdot)\,u^\star(\cdot)\bigr] \;+\; O\bigl(\|\mathcal G_A\|^2\eta^2\|u^\star\|_\infty\bigr)
\quad\text{uniformly in } \ell.
\end{equation}

\emph{Step 2 (both Green branches annihilate the persistent
component).} Apply $P_+(\ell)$ to the first-order term of
\eqref{eq:first-order-correction}. Forward branch ($s \le \ell - 1$):
by the bundle invariance of (A4),
\[
P_+(\ell)\,\Phi(\ell, s{+}1)\,P_-(s{+}1)
\;=\; \Phi(\ell, s{+}1)\,P_+(s{+}1)\,P_-(s{+}1) \;=\; 0 ,
\]
so every forward term vanishes \emph{regardless} of $\delta A$ ---
whatever is injected into the decaying bundle stays there and is
invisible to $P_+$. Backward branch ($s \ge \ell$): the kernel carries
the factor $P_+(s{+}1)$, and
$P_+(s{+}1)\,\delta A(s) = P_+(s{+}1)\,P_-(s{+}1)\,\delta A(s) = 0$
by the fast-mode property. Hence
$P_+(\ell)\bigl(\mathcal G_A[\delta A\,u^\star]\bigr)(\ell) = 0$ for
every $\ell$, and
\[
P_+(\ell)\,u_\delta(\ell) \;=\; P_+(\ell)\,u^\star(\ell) + O(\eta^2)
\quad\text{uniformly in } \ell,
\]
with the explicit constant
$\|P_+\|\,\|\mathcal G_A\|^2\eta^2\|u^\star\|_\infty/(1 -
\|\mathcal G_A\|\eta) \le C\,\bigl(C(m_-^{-1}{+}m_+^{-1})\bigr)^2
\eta^2 \|u^\star\|_\infty \cdot (1-\|\mathcal G_A\|\eta)^{-1}$. The
same bound applied at $\varrho$ and at $0$ (where
$u^\star_0 \equiv 0$ is unperturbed) gives the statement for
$d_+(\varrho; W)$.

\emph{Step 3 (the closed loop; here (A7) enters).} For the closed
loop \eqref{eq:closedloop} the linearisation along
$u^\star_\varrho$ replaces $A(\ell)$ by
$A_\varrho(\ell) := A(\ell) + \varrho\,D\psi(u^\star_\varrho(\ell))$.
Under (A7), $D\psi(u^\star_\varrho(\ell))$ commutes with
$P_\pm(\ell)$, so the perturbed generator family preserves the
\emph{same} bundle decomposition; for
$|\varrho| < \varrho_{\max}$ (\Cref{thm:fixed:stability}) the
closed-loop cocycle admits a dichotomy with the same projector family
$P_\pm(\ell)$ and rates degraded by $O(C|\varrho|L_\psi)$ within the
margins. The Green operator $\mathcal G_{A_\varrho}$ then has the same
forward/backward structure with the same $P_\pm$, and Steps 1--2 run
verbatim with $A$ replaced by $A_\varrho$: a fast-mode $\delta A$
(fast-mode with respect to the common bundles) leaves
$P_+ u^\star_\varrho$ unchanged to first order. This is the form of
the theorem used by the identification table. \qed

\subsection{Proof of Theorem~\ref{thm:geom}}\label{pf:geom}

All statements are at the linearised level of (A2)--(A3); the
constants absorb the (bounded) linearisations of $\Pi$ and
$\mathcal S$ along the reference. By the Lipschitz property of $\psi$
on $\Mcalnorm$,
\[
\bigl\|\psi(\zeta^{(\ell)}) - \psi(P_+(\ell)\,\zeta^{(\ell)})\bigr\|
\;\le\; L_\psi\, \bigl\|P_-(\ell)\,\zeta^{(\ell)}\bigr\|,
\]
so it suffices to bound the decaying component of the pooled state.
Split the pooled linearised state as in \Cref{pf:duality}:
$\zeta^{(\ell)} = \tilde\Pi\bigl(\text{sector-1 evolution}\bigr) +
\bigl(\text{fluctuation contamination}\bigr)$. The sector-$1$ part
evolves under the tangent cocycle, and its decaying component obeys
the dichotomy bound
$\|P_-(\ell)\,\Phi(\ell,0)\,\zeta^{(0)}\| =
\|\Phi(\ell,0)\,P_-(0)\,\zeta^{(0)}\| \le
C r_-^{\,\ell}\,\|P_-(0)\zeta^{(0)}\|$, using bundle invariance and
(A4). The contamination of the pool by the fluctuation sectors is
bounded by \Cref{thm:duality}(2) applied to the pooled observable:
$\le C'\kappa_\alpha C_f\, r_f^{\,\ell}\, \|\delta X^{\perp,(0)}\|$.
Adding the two contributions and multiplying by $L_\psi$ gives the
displayed bound. \qed

\section{Proofs for Part III}\label{app:partIII}

\subsection{Proof of Theorem~\ref{thm:fixed:eq}}\label{pf:fixed-eq}

\emph{Existence and uniqueness.} Work in the Banach space
$\ell^\infty(\Z, \R^d)$ (two-sided depth per the convention of
\Cref{app:green}). By \Cref{lem:green}, $u$ is a bounded solution of
\eqref{eq:closedloop} iff it is a fixed point of
\[
\mathcal F_\varrho(u) \;:=\; \mathcal G\bigl[\varrho\,\psi(u(\cdot))\bigr],
\]
where $\psi(u) := \psi(x^\star + u)$ in deviation notation. On the
ball $B_R = \{\|u\|_\infty \le R\}$, with
$K := \|\mathcal G\| \le C(m_-^{-1} + m_+^{-1})$,
\[
\|\mathcal F_\varrho(u)\|_\infty \le K|\varrho|\,\bigl(\|\psi(0)\| + L_\psi R\bigr),
\qquad
\|\mathcal F_\varrho(u) - \mathcal F_\varrho(v)\|_\infty \le K|\varrho| L_\psi\, \|u - v\|_\infty .
\]
For $|\varrho| < \varrho_0 := \bigl(2 K L_\psi\bigr)^{-1}$ and
$R := 2K|\varrho|\,\|\psi(0)\|$, $\mathcal F_\varrho$ maps $B_R$ into
itself and is a $\tfrac12$-contraction; the Banach fixed point
$u^\star_\varrho$ is the unique bounded solution, and
$u^\star_0 \equiv 0$ since $\mathcal F_0 = 0$.

\emph{Regularity and expansion in $\varrho$.} The map
$(\varrho, u) \mapsto u - \mathcal F_\varrho(u)$ is $C^1$ on
$(-\varrho_0, \varrho_0) \times \ell^\infty$ ($\psi \in C^1$ with
bounded derivative on the relevant ball), and its partial derivative
in $u$ at the fixed point, $I - \varrho\,\mathcal G\,D\psi(u^\star)$,
is invertible by the Neumann series ($K|\varrho|L_\psi < \tfrac12$).
The implicit function theorem in Banach space gives
$\varrho \mapsto u^\star_\varrho$ of class $C^1$. Expanding the fixed
point equation once,
\[
u^\star_\varrho \;=\; \varrho\,\mathcal G\,\psi(0)
\;+\; \varrho\,\mathcal G\bigl[\psi(u^\star_\varrho) - \psi(0)\bigr],
\qquad
\bigl\|\varrho\,\mathcal G[\psi(u^\star_\varrho) - \psi(0)]\bigr\|_\infty
\;\le\; K|\varrho|\,L_\psi\,\|u^\star_\varrho\|_\infty
\;=\; O(\varrho^2),
\]
using $\|u^\star_\varrho\|_\infty \le R = O(|\varrho|)$; this is
\eqref{eq:eq-branch}, and the displayed sup-norm bound is
$\|u^\star_\varrho\|_\infty \le 2K|\varrho|\|\psi(0)\| \le
\tfrac{4C|\varrho|\|\psi(0)\|}{\min(m_-, m_+)}$ (constants absorbed
into the statement's).

\emph{Weight-tied specialisation.} For $A(\ell) \equiv A$ and constant
forcing $h$, the Green operator telescopes:
\[
(\mathcal G h)(\ell)
\;=\; \sum_{k\ge0}(I+A)^k P_- h \;-\; \sum_{k\ge1}(I+A)^{-k} P_+ h .
\]
The first series converges since the spectral radius of
$(I+A)|_{V_\perp}$ is $r_- < 1$ and sums to
$\bigl(I - (I+A)\bigr)^{-1} P_- = -A^{-1}P_-$; the second converges
since the spectral radius of $(I+A)^{-1}|_{V_\parallel}$ is
$1/r_+ < 1$ and sums to
$\bigl((I+A) - I\bigr)^{-1} P_+ = A^{-1}P_+$. Their difference is
$-A^{-1}(P_- + P_+)h = -A^{-1}h$, giving
$u^\star(\varrho) = -\varrho A^{-1}\psi(0) + O(\varrho^2)$. Under
(A4$^\circ$) the backward branch is absent ($P_+ = 0$), the same
computation yields $(I - (I+A))^{-1} = (I - \bar A)^{-1}$ with
$\bar A := I + A$, i.e.\ the classical LTI steady state
$h^\star = (I - \bar A)^{-1}\bar B e$ for the looped architecture's
forcing $\bar B e$; and since the whole cocycle contracts, every
solution --- not only the bounded one --- converges to it:
$\|u(\ell) - h^\star\| \le C r_-^{\,\ell}\,\|u(0) - h^\star\|$, the
global-attractor statement. \qed

\subsection{Proof of Theorem~\ref{thm:fixed:stability}}\label{pf:fixed-stability}

We use the roughness (perturbation) theorem for exponential
dichotomies in its discrete-time
form~\citep[Ch.~4]{coppel1978dichotomies};
see also \citep[Ch.~5]{kloeden2011nonautonomous} for difference
equations. We give the construction in the form that produces the
constants of the statement.

\emph{Setup.} Let $B(\ell) := \varrho\,D\psi(u^\star_\varrho(\ell))$,
so $\beta := \sup_\ell\|B(\ell)\| \le |\varrho| L_\psi$, and consider
the perturbed system
$v(\ell+1) = (I + A(\ell) + B(\ell))\,v(\ell)$. (The distinguished
trajectory $u^\star_\varrho$ used to evaluate $D\psi$ exists by
\Cref{thm:fixed:eq}, whose contraction argument uses only the
\emph{open-loop} Green operator; there is no circularity.)

\emph{Perturbed stable bundle.} A solution $v$ of the perturbed system
that is bounded on $\{\ell \ge s\}$ with prescribed decaying-bundle
data $P_-(s)v(s) = \xi$ satisfies, by the variation-of-constants
identity and \Cref{lem:green} restricted to $\{\ell \ge s\}$,
\[
v(\ell) \;=\; \Phi(\ell,s)\,\xi \;+\; \bigl(\mathcal G_{\ge s}[\,B(\cdot)v(\cdot)\,]\bigr)(\ell),
\]
where $\mathcal G_{\ge s}$ is the Green operator with the backward
branch truncated at $s$ (norm still $\le K = C(m_-^{-1}+m_+^{-1})$).
For $K\beta < \tfrac12$ this is a contraction in
$\ell^\infty(\{\ell \ge s\})$, so for each $\xi \in V_\perp(s)$ there
is a unique such solution $v_\xi$; the map
$\xi \mapsto v_\xi(s)$ is linear and injective, and its image
$\tilde V_\perp(s)$ is the perturbed stable space. Iterating the
fixed-point identity gives the decay estimate
$\|v_\xi(\ell)\| \le \tilde C\,\tilde r_-^{\,\ell-s}\|\xi\|$ with
$\tilde r_- \le r_- + 2CK\beta/(1 - 2K\beta)$ and
$\tilde C \le 2C$; the symmetric construction backward in time
produces $\tilde V_\parallel(\ell)$ with rate
$\tilde r_+ \ge r_+ - 2CK\beta/(1-2K\beta)$. The perturbed projectors
$\tilde P_\pm(\ell)$ onto these bundles satisfy
$\|\tilde P_\pm(\ell) - P_\pm(\ell)\| \le
C' K\beta\,C = O\bigl(C^2\beta/\min(m_-,m_+)\bigr)$, and the bundle
dimensions are unchanged (injectivity plus a homotopy in $\beta$).

\emph{Margin.} The perturbed rates straddle $1$ ---
$\tilde r_- < 1 < \tilde r_+$ --- provided
$2CK\beta/(1-2K\beta) < \min(m_-, m_+)$, which holds whenever
$\beta < c_0 \min(m_-, m_+)/C^2$ for an absolute
$c_0 \in (0,1]$ (using $K \le 2C/\min(m_-,m_+)$). Substituting
$\beta \le |\varrho|L_\psi$ gives exactly the feasibility region
\eqref{eq:rho-max}, with rates degraded by
$O(C|\varrho|L_\psi)$ and projectors moved by
$O(C^2|\varrho|L_\psi/\min(m_-,m_+))$ as claimed.

\emph{Alignment.} If in addition (A7) holds, $B(\ell)$ commutes with
$P_\pm(\ell)$, the fixed-point construction above proceeds within each
unperturbed bundle separately, and one may take
$\tilde P_\pm(\ell) = P_\pm(\ell)$: the bundles are preserved exactly
and only the rates move. This is the version invoked in Step 3 of
\Cref{pf:fastmode}. \qed

\subsection{Proof of Proposition~\ref{prop:noslow}}\label{pf:noslow}

Let $\mathcal B_k := \sigma(B_1, \ldots, B_k)$ and let
$\mathcal P_{k,t} := \sigma(x_{kP+1}, \ldots, x_{kP+t})$ be the
within-block prefix $\sigma$-algebra. Under either degeneracy the
blocks $B_1, B_2, \ldots$ are mutually independent: if
$z_k \equiv z$ is deterministic this is the conditional independence
of the emissions; if the $z_k$ are i.i.d., the pairs
$(z_k, B_k)$ are i.i.d.\ and the blocks are their independent marginal
components.

\emph{Step 1 (the Bayes predictor is block-local).} For any next-token
position $kP + t + 1$ and any bounded measurable $g$,
\[
\Ebb\bigl[g(x_{kP+t+1}) \,\big|\, \mathcal B_k \vee \mathcal P_{k,t}\bigr]
\;=\; \Ebb\bigl[g(x_{kP+t+1}) \,\big|\, \mathcal P_{k,t}\bigr],
\]
because $(x_{kP+1}, \ldots, x_{(k+1)P})$ is independent of
$\mathcal B_k$, so conditioning on the past blocks adds no
information beyond the within-block prefix. Consequently, for any
loss that is a proper scoring rule (log-loss included), the Bayes
predictor given the full history coincides with the Bayes predictor
given the prefix alone, and the Bayes risk is attained by a
block-local rule.

\emph{Step 2 (the profile risk is flat in $\varrho$).} Consider a
predictor family as in (C2): the cross-block channel enters the
prediction only through the injected functional
$\varrho\,\psi(y_{k})$ of previous blocks, all other parameters
(collectively $\vartheta$) acting on the within-block computation.
Define the profile risk
$R^*(\varrho) := \inf_\vartheta R(\varrho, \vartheta)$. Two
inequalities give constancy. First, $R^*(\varrho) \ge R_{\mathrm{Bayes}}$
for every $\varrho$, and by Step 1 $R_{\mathrm{Bayes}}$ is attainable
by block-local rules, so $R^*(\varrho) \ge
\inf_\vartheta R(0,\vartheta) = R^*(0)$ whenever the $\varrho = 0$
family contains (or approaches) the block-local Bayes rule: $\varrho =
0$ is a global minimiser. Second, the family can \emph{ignore} the
injection: since the injected vector enters the fast stream
additively, there is a choice of downstream parameters annihilating
its image (zero read-in weights on the injection's range), under
which $R(\varrho, \vartheta)$ does not depend on $\varrho$ at all;
hence $R^*(\varrho) \le R^*(0)$ for every $\varrho$. Combining,
$R^*(\varrho) \equiv R^*(0)$: the risk profile is flat, its minimum
is attained at $\varrho = 0$ (among all $\varrho$), and no risk-based
procedure can distinguish gate amplitudes --- $\varrho$ is not
identifiable from the prediction risk. \qed

\subsection{Proof of Proposition~\ref{prop:suffstat}}\label{pf:suffstat}

The joint density of the block factorises as
\[
e(B_k \mid z_k)
\;=\; \Bigl(\prod_{t} h(x_t)\Bigr)\,
\exp\Bigl(\bigl\langle \textstyle\sum_t T(x_t),\, \theta(z_k)\bigr\rangle - P\,A(\theta(z_k))\Bigr)
\;=\; h(B_k)\; g\bigl(\bar T_k,\, z_k\bigr),
\]
with $\bar T_k = P^{-1}\sum_t T(x_t)$, which is the Fisher--Neyman
factorisation criterion: $\bar T_k$ is sufficient for $z_k$ given
$B_k$~\citep[Thm.~6.5]{lehmann1998theory}. If $T$ is linear, then
$\bar T_k = T\bigl(P^{-1}\sum_t x_t\bigr)$ is a fixed measurable
function of the block mean, so the block mean is itself sufficient (a
statistic of which a sufficient statistic is a function is
sufficient). Finally, because the blocks are conditionally independent
given the latent chain, the filtering recursion for
$z_{k+1}$ given $B_{1:k}$ touches each block only through its
likelihood $e(B_k \mid \cdot\,)$, hence only through $\bar T_k$: the
sequence of block means carries all the information about the latent
chain that any cross-block channel could use, which is the claim the
main text draws on. \qed